 \documentclass[preprint,12pt,3p,review]{elsarticle}

\usepackage[colorlinks=true]{hyperref}

\usepackage{amssymb}

\usepackage[labelsep=period,singlelinecheck=off,textformat=period]{caption}
\usepackage{subcaption}
\usepackage{fancyhdr} 
\usepackage{epsfig}
\usepackage{timesmt}
\usepackage{mathtools,nccmath}
\usepackage{mathrsfs}
\usepackage{wrapfig}
\usepackage{soul}

\usepackage[noabbrev]{cleveref}
\creflabelformat{equation}{#2\textup{#1}#3}
\creflabelformat{equation}{#2(#1)#3}
\crefrangelabelformat{equation}{#3(#1)#4 to #5(#2)#6}
\labelcrefformat{subequation}{#2(#1)#3}
\labelcrefrangeformat{subequation}{#3(#1)#4 to #5(#2)#6}

\usepackage{pdflscape}
\usepackage{afterpage}
\usepackage{float}
\usepackage{url}
\usepackage{enumitem}
\usepackage{accents}
\usepackage{cuted}
\usepackage{dirtytalk}
\usepackage{multirow}
\usepackage{multicol}
\usepackage{booktabs}
\usepackage{array}
\usepackage{makecell}
\usepackage{colortbl}
\newcolumntype{n}{>{\columncolor{blue!5}}c}

\usepackage{pbox}
\usepackage[ruled,commentsnumbered,linesnumbered,boxed]{algorithm2e}
\SetArgSty{textnormal}
\SetKwComment{Comment}{/* }{ */}

\usepackage[flushleft]{threeparttable}
\usepackage{algpseudocode}
\usepackage{setspace}
\usepackage{courier}
\usepackage{psfrag,pstool}

\Crefname{equation}{Eq.}{Eqs.}
\Crefname{figure}{Fig.}{Figs.}
\Crefname{tabular}{Tab.}{Tabs.}
\crefname{algocf}{alg.}{algs.}
\Crefname{algocf}{Algorithm}{Algorithms}
\Crefname{fct}{Fact}{Facts }

\usepackage{ulem}
\newcommand\greensout{\bgroup\markoverwith{\textcolor{green!40!black}{\rule[0.5ex]{2pt}{1pt}}}\ULon}
\newcommand{\stkout}[1]{\ifmmode\text{\greensout{\ensuremath{#1}}}\else\greensout{#1}\fi}

\usepackage{amsmath}
\usepackage{amssymb}
\usepackage{amsthm}
\usepackage{mathrsfs}

\usepackage{bm}
\usepackage{mathtools}
\usepackage{nicefrac}

\theoremstyle{plain}
\newtheorem{theorem}{Theorem}[section] 
\newtheorem{lemma}{Lemma}[section] 
\newtheorem{remark}{Remark}[section]

\newtheorem{assumption}{Assumption}[section]

\theoremstyle{definition}

\newcommand{\ie}{i.e., }

\usepackage{pgfplots}
\usepackage{tikz}
\usetikzlibrary{shapes.misc}
\usetikzlibrary{arrows,arrows.meta}
\tikzstyle{block} = [draw,rectangle,thick,minimum height=2em,minimum width=2em]
\tikzstyle{sum} = [draw,circle,inner sep=0mm,minimum size=2mm]
\tikzstyle{connector} = [->,thick]
\tikzstyle{line} = [thick]
\tikzstyle{branch} = [circle,inner sep=0pt,minimum size=1mm,fill=black,draw=black]
\tikzstyle{guide} = []
\tikzset{>=latex}

\newcommand{\x}{{\bm{x}}}
\newcommand{\y}{{\bm{y}}}
\newcommand{\z}{{\bm{z}}}
\newcommand{\X}{{\bm{X}}}
\newcommand{\Y}{{\bm{Y}}}
\newcommand{\Z}{{\bm{Z}}}

\newcommand{\s}{{\bm{s}}}

\newcommand{\Ro}{\mathbb{R}}

\newcommand{\dby}{\Omega_Y}

\newcommand{\lip}{\text{Lip}}
\newcommand{\lipx}{\text{Lip}_X}
\newcommand{\lipxe}{\text{Lip}^\prime_X}

\newcommand{\loss}{\mathcal{L}}     	 
\newcommand{\g}{{\bm{g}}}  

\newcommand{\Nz}{{N_Z}}
\newcommand{\Nx}{{N_X}}
\newcommand{\Ny}{{N_Y}}

\newcommand{\ud}{\text{d}}

\newcommand{\argmax}[1]{\underset{#1}{\text{argmax}} \ }    
\newcommand{\argmin}[1]{\underset{#1}{\text{argmin}} \ }

\newcommand{\ex}[2]{\mathbb{E}_{#1}\left[ #2\right]}

\newcommand{\RNum}[1]{\uppercase\expandafter{\romannumeral #1\relax}}

\renewcommand{\mathbf}[1]{\bm{#1}}

\graphicspath{{./Figures/}}

\usepackage{lineno}

\begin{document}

\begin{frontmatter}



\title{Solution of physics-based inverse problems using conditional generative adversarial networks with full gradient penalty}


\author[label1]{Deep Ray}
\author[label2]{Javier Murgoitio-Esandi}
\author[label2]{Agnimitra Dasgupta}
\author[label2]{Assad A. Oberai}

\affiliation[label1]{organization={{Department of Mathematics,
			University of Maryland}},
	city={College Park},
	postcode={20742}, 
	state={Maryland},
	country={USA}}
\affiliation[label2]{organization={{Aerospace and Mechanical Engineering Department,
		University of Southern California}},
        city={Los Angeles},
        postcode={90089}, 
        state={California},
        country={USA}}

\begin{abstract}
The solution of probabilistic inverse problems for which the corresponding forward problem is constrained by physical principles is challenging. This is especially true if the dimension of the inferred vector is large and the prior information about it is in the form of a collection of samples. In this work, a novel deep learning based approach is developed and applied to solving these types of problems. The approach utilizes samples of the inferred vector drawn from the prior distribution and a physics-based forward model to generate training data for a conditional Wasserstein generative adversarial network (cWGAN). The cWGAN learns the probability distribution for the inferred vector conditioned on the measurement and produces samples from this distribution. The cWGAN developed in this work differs from earlier versions in that its critic is required to be 1-Lipschitz with respect to both the inferred and the measurement vectors and not just the former. This leads to a loss term with the full (and not partial) gradient penalty. It is shown that this rather simple change leads to a stronger notion of convergence for the conditional density learned by the cWGAN and a more robust and accurate sampling strategy. Through numerical examples it is shown that this change also translates to better accuracy when solving inverse problems. The numerical examples considered include illustrative problems where the true distribution and/or statistics are known, and a more complex inverse problem motivated by applications in biomechanics.

\end{abstract}


%


\end{frontmatter}

\section{Introduction}
\label{sec:introduction}

Inverse problems arise in many fields of engineering and science. Typically, they involve recovering the input fields or parameters in a system given noisy measurements of its response. In contrast to direct problems, where one wishes to determine the response given the input parameters, inverse problems are notoriously hard to solve. This is because they are often ill-posed in the sense of Hadamard \cite{hadamard1902problemes} and involve measurements that are corrupted by noise, while their solution often requires injecting some prior information or beliefs regarding the inferred parameters or fields. 


The challenges described above can be addressed by working with a probabilistic version of the inverse problem where the input and response vectors/fields are treated as random vectors /processes. In this manuscript, we refer to the input vector as the vector of the parameters to be inferred, or the inferred vector for short. 
In the probabilistic framework, the solution of an inverse problem transforms to characterizing the conditional probability distribution of the inferred vector conditioned on a measurement of the response; this  distribution is known as the posterior distribution. Using Bayes' rule, the posterior distribution can be expressed (up to a multiplicative constant) as the product of a marginal distribution for the inferred vector with the conditional distribution for the measurement given the inferred parameters \cite{dashti2017bayesian}. The former is replaced by a prior distribution for the parameters to be inferred, and the latter leads to the likelihood term, which is determined by the model for the noise in the measurement and the forward map that links the parameters to the response. Once the probabilistic inverse problem has been framed, a method like Markov Chain Monte Carlo (MCMC) \cite{chib2001markov,brooks2011handbook} may be applied to sample from the posterior distribution; these samples help approximate the posterior distribution that may otherwise be analytically intractable~\cite{liu2001monte}. The steady state of the Markov chain leads to samples that are drawn from the desired conditional distribution for the input parameters. 



While the MCMC method and its variants can be used to solve some probabilistic inverse problems, their efficacy diminishes under several conditions. These include problems where the dimension of the inferred vector becomes large (greater than $\approx$ 50, for example) --- the so-called \textit{curse of dimensionality} \cite{gelman1997weak,cui2016dimension,huijser2015properties}. It also includes problems where  prior information about the  parameters to be inferred cannot be expressed in terms of a simple probability distribution and is determined through samples. Finally, some variants of the MCMC method with improved convergence properties, like the Hamiltonian Monte Carlo (HMC) method \cite{neal2011mcmc,betancourt2017conceptual}, rely on computing the derivative of the forward map with respect to the parameters. In cases where the forward map is implemented through a complex legacy code, computing this derivative is not feasible and these methods cannot be used. 

Recently, there has been a growing interest in utilizing novel machine and deep learning algorithms to address the challenges described above. Most of these methods first consider samples of the parameters to be inferred derived from a prior distribution and utilize the forward map and a model for noise to generate corresponding samples of the measurement vector. Thereafter, they use this dataset to train algorithms for sampling the conditional distribution for the parameters given the measured response. 
One such class of methods that utilize deep neural networks to map the response vector to the vector of parameters, is referred to as Bayesian Neural Networks~(BNNs)~\cite{pml2Book,blundell2015weight,jospin2022hands,lampinen2001bayesian}. In these networks, in addition to the input and output vectors, the network weights are also treated as random variables. Thereafter, Bayes' rule is applied to derive an expression for the posterior distribution of the network parameters (weights and biases) given the training data. Then, variational inference is used to derive an approximation for the posterior density of the network parameters. Once this density is known, multiple realizations of the network parameters are sampled and the resulting networks are applied to the same measured response to yield an ensemble of inferred vectors. The performance of BNNs is limited by the approximations that are typically incurred to make variational inference tractable, which include assuming mixture models for the weights and selecting a small subset of the network weights as stochastic~\cite{fortuin2022priors,tran2022all}.


Another class of deep learning algorithms utilize Generative Adversarial Networks~(GANs)~\cite{goodfellow2014gans} to learn and then sample from the posterior distribution. GANs are generative models that learn a reduced-order representation of a probability distribution from its samples and then efficiently generate samples from it. GANs comprise of two networks: a generator and a critic. The generator network maps low-dimensional vectors drawn from a simple probability distribution to the output domain, while the critic maps samples from the output domain to a scalar. For the Wasserstein~GAN~\cite{arjovsky2017,gulrajani2017improved}, under suitable assumptions, it can be shown that the probability distribution of the samples produced by the trained generator converges to the true distribution (from which the training samples are drawn) in the Wasserstein-1 distance. This is achieved by imposing a 1-Lipschitz constraint on the critic. 
In \cite{patel2021ganb}, a WGAN was used to learn the prior distribution of the 
inferred vector, and then used in conjunction with the HMC method to learn the posterior conditional distribution. While this approach reduced the dimension of the underlying inverse problem by mapping it to the low-dimensional latent space of the WGAN, it still required the use of an HMC method to sample from the posterior distribution. 

In contrast to the approach described above, conditional GANs \cite{mirza2014conditional} can be used to approximate the posterior distribution directly. \citet{adler2018deep} proposed a conditional WGAN (cWGAN) that aims to directly learn the posterior distribution, and used it to improve the results of an inverse problem in medical imaging. \citet{ray2022} improved the architecture of cWGAN and applied it to solve complex inverse problems motivated by physics-driven applications. \citet{kovachki2020conditional} proposed the Monotone GAN (MGAN) to solve inverse problems. This method differs from the cWGAN in several ways. First, it uses a latent space whose dimension is the same as that of the inferred  vector, whereas the cWGAN typically uses a latent space of reduced dimension. Second, it relies on block triangular transformations to map the latent space to the space of the inferred vector, whereas a cWGAN uses a U-NET \cite{ronneberger2015u} architecture.  Finally, it requires the generator that transforms the latent space to the space of inferred parameters to be monotonic, whereas the cWGAN imposes a 1-Lipschitz condition on the critic and not the generator.

The cWGAN proposed by \citet{adler2018deep} uses a critic network that is 1-Lipschitz with respect to the inferred vector only. The 1-Lipschitz constraint is implemented weakly through a gradient penalty term, and the gradient is computed with respect to the inferred vector. 
Herein, we refer to this as Partial-GP. As a result of the partial 1-Lipschitz constraint being imposed on the critic, the Kantorovich-Rubinstein duality (see \cite{villani2008optimal} for example) can be invoked to show that the critic for the cWGAN with Partial-GP belongs to the space of functions that can only be optimized to approximate the Wassertstein-1 distance between distributions of the inferred vector. Based on this, \citet{adler2018deep} develop a training objective function for the cWGAN that aims to reduce the averaged Wasserstein-1 distance between the true posterior distribution and the one induced by the cWGAN; the average is taken over all possible values of the measurement vector. This does not necessarily mean that the distribution of the inferred vector conditioned on any one particular value of the measurement vector will converge to the true posterior distribution even if the critic and generator networks are sufficiently expressive, and are perfectly trained.

In this work we improve upon the cWGAN with Partial-GP \cite{adler2018deep}, with a novel cWGAN formulation which differs from the original formulation in one critical aspect. In our approach, we require the critic network to be a 1-Lipschitz function of all its inputs --- the inferred and measurement vector. Therefore,  the crucial algorithmic difference that arises is that the gradient must be computed with respect to both arguments of the critic as opposed to just the inferred vector for Partial-GP. Remarkably, this small change in the implementation has a significant effect on the properties of the final network. Due to the critic being 1-Lipschitz with respect to its entire input argument, we can show, via the Kantorovich-Rubinstein duality, that the critic in the cWGAN belongs to the space of functions that can minimize the Wasserstein-1 between distributions of the joint random vector containing both the inferred variable and the measurements. As a result, we show that the generator of a cWGAN that is perfectly trained with the same training objective as~\cite{adler2018deep} will converge to the true conditional distribution for the inferred vector for any measurement vector.


Additionally, the convergence guarantees of the proposed cWGAN with Full-GP allow us to develop a new evaluation strategy where the cWGAN is evaluated at measurement values obtained by perturbing the actual measurement vector with samples from a zero-mean Gaussian distribution that has a minute variance. We show that this new evaluation strategy may lead to improved performance of cWGANs. Further, we study the behavior of the proposed cWGAN with Full-GP on a suite of illustrative examples and its ability to approximate the joint distribution and the posterior distribution. We also apply the proposed cWGANs to two challenging inverse problems arising in inverse heat conduction and elastography \cite{barbone2010review} that help establish its efficacy in solving physics-based inverse problems.

The format of the remainder of this manuscript is as follows. In the following section we introduce the mathematical background required to develop and analyze the new cWGAN with Full-GP. In \Cref{sec:cwgan_fullGP},  we describe how the Full-GP cWGAN is trained and then used to solve a probabilistic inverse problem. In Section \Cref{sec:convergence}, we perform a detailed analysis of the convergence of the conditional density generated by Full-GP cWGAN with increasing number of trainable generator parameters. We also compare it with the corresponding analysis for the cWGAN with Partial-GP. In Section \Cref{sec:results}, we apply the Full-GP and Partial-GP cWGAN networks to inverse problems with increasing complexity, and evaluate their ability to recover the true posterior distribution for cases where it is known. We observe that the Full-GP cWGAN performs better. We end with conclusions in \Cref{sec:conclusion}.

\section{Background}

\subsection{Problem setup}

Consider the following measurement model
\begin{equation}\label{eqn:forward}
	\y = \mathcal{F}(\x; \bm{\eta})
\end{equation}
where $\mathcal{F}$ is the forward model/process, $\x \in \Ro^{\Nx}$ is some input field to the forward model, while $\y \in \Ro^\Ny$ is the corresponding output/measured field of the model corrupted by some measurement noise $\bm{\eta}$. Then our goal is to solve the inverse problem of reconstructing/inferring $\x$ given some noisy measurement $\y$. In the probabilistic setting, necessary for carrying out Bayesian inference, the inferred field $\x$ and the measured field $\y$ are modeled as random variables $\X$ and $\Y$, respectively. Let $\mu_{\X\Y}$ be the joint probability measure corresponding to the pair of random variable $(\X,\Y)$, with marginal measures $\mu_\X$, $\mu_\Y$. Then, given a realization $\Y = \hat{\y}$, we wish to approximate the conditional measure $\mu_{\X|\Y}(.|\hat{\y})$ and efficiently draw samples from it in order to evaluate useful conditional statistics.

\subsection{Conditional GANs}
Conditional generative adversarial networks (cGANs) \cite{mirza2014conditional} can be used to learn conditional probability measures. Typically, cGANs comprise two deep neural networks, a generator $\g$ and a critic $d$. The generator is a mapping given by
\begin{equation}\label{eqn:generator}
	\g: \Ro^{\Nz} \times \Ro^\Ny \rightarrow \Ro^\Nx, \quad \g:(\z,\y) \mapsto \x
\end{equation}
where $\z \in \Ro^\Nz$ is a realization of the latent random variable $\Z$ with the measure $\mu_\Z$. The latent measure is chosen to be simple, such as a multidimensional Gaussian, to make it easy to sample from. The generator can be interpreted as a mapping that, given a $\y \sim \mu_\y$, can generate `\textit{fake}' samples of the inferred field from the push-forward measure $\mu^{\g}_{\X|\Y} = \g(.,\y)_\# \mu_\Z$. 
The critic, on the other hand, is given by the mapping
\begin{equation}\label{eqn:critic}
	d: \Ro^{\Nx} \times \Ro^\Ny \rightarrow \Ro, \quad d:(\x,\y) \mapsto r
\end{equation}
whose role is to distinguish between paired samples $(\x,\y)$ drawn from the true joint measure $\mu_{\X\Y}$ and the fake/generated joint measure $\mu^{\g}_{\X\Y} = \mu^{\g}_{\X|\Y} \mu_\Y$.

In a conditional Wasserstein GAN (cWGAN) \cite{adler2018deep}, a variant of cGAN, both the generator and critic networks are trained in an adversarial manner using the objective `\textit{loss}' function
\begin{equation}\label{eqn:objective}
	\begin{aligned}
		\loss(\g,d) &= \ex{\mu_{\X\Y}}{d(\X,\Y)} - \ex{\mu_\Y}{\ex{\mu^{\g}_{\X|\Y}}{d(\X,\Y)}}\\
		&= \ex{\mu_{\X\Y}}{d(\X,\Y)} - \ex{\mu^{\g}_{\X\Y}}{d(\X,\Y)},
	\end{aligned}    
\end{equation}
and solving the minmax problem
\begin{equation}\label{eqn:minmax}
	\begin{aligned}
		g^*, d^* = \argmin{\g} \argmax{d}\loss(\g,d).
	\end{aligned}    
\end{equation}

\section{Conditional WGAN with Full-GP}\label{sec:cwgan_fullGP}

In this work, we propose a cWGAN model with Full-GP which, unlike the original cWGAN proposed by \citet{adler2018deep}, imposes a 1-Lipschitz constraint on the critic $d$ with respect to its entire argument $(\x, \y)$. Let us denote by $\lip$ the 1-Lipschitz functions on $\Ro^{\Nx} \times \Ro^{\Ny}$. If we assume that the maximization in \eqref{eqn:minmax} is over $d \in \lip$, then we can show by the Kantorovich–Rubinstein duality argument \cite{villani2008optimal} that 
\begin{equation}\label{eqn:minmax2}
	\begin{aligned}
		d^*(\g) &= \argmax{d \in \lip} \loss(\g,d) = W_1(\mu_{\X\Y},\mu^\g_{\X\Y}),\\ 
		\g^* &= \argmin{\g} \loss(\g,d^*(\g)) = \argmin{\g} W_1(\mu_{\X\Y},\mu^\g_{\X\Y}),
	\end{aligned}    
\end{equation}
where $W_1$ is the Wasserstein-1 distance defined on the space of joint probability measures on $\Ro^{\Nx} \times \Ro^{\Ny}$. In practice, the 1-Lipschitz constraint on the critic can be imposed by augmenting the loss function by a gradient penalty term \cite{gulrajani2017improved} when training the critic, i.e., by solving
\begin{equation}\label{eqn:max_with_GP}
	d^*(\g) = \argmax{d \in \lip} \left[\loss(\g,d) - \lambda \mathcal{G}\mathcal{P} \right] 
\end{equation}
where $\lambda > 0$ is a hyper-parameter which appropriately weights the penalty term. We use the following penalty term:
\begin{equation}\label{eqn:full_GP}
	\mathcal{G}\mathcal{P} = \ex{\delta \sim \mathcal{U}(0,1)}{(\|\nabla d( \bm{h}(\x,\y,\z,\delta),\y)\|_2 - 1)^2},
\end{equation}
where $\mathcal{U}(0,1)$ denotes the uniform distribution on $[0,1]$, $\nabla$ denotes the gradient with respect to both its arguments, and 
\begin{equation}\label{eqn:blending}
	\bm{h}(\x,\y,\z,\delta) = \delta \x + (1-\delta) \g(\z,\y).
\end{equation}
The gradient penalty term in \Cref{eqn:full_GP} imposes a 1-Lipschitz constraint on the critic with respect to both its arguments $\x$ and $\y$. Using this modified penalty term is crucial to ensure that solving minmax problem is equivalent to minimizing the $W_1$ distance between the true joint measure $\mu_{\X\Y}$ and the generated joint $\mu^\g_{\X\Y}$. This in turn leads to an alternative analysis of (weak) convergence to the true conditional measure $\mu_{\X|\Y}$ which is described in Section \ref{sec:convergence}. 

\begin{remark}
    In the original cWGAN model \cite{adler2018deep}, the penalty term was chosen to constrain the critic to be 1-Lipschitz only with respect to the first argument $\x$. More precisely, \citet{adler2018deep} use the following penalty term: 
\begin{equation}\label{eqn:part_GP}
	\mathcal{G}\mathcal{P} = \ex{\delta \sim \mathcal{U}(0,1)}{(\|\partial_{1}d( \bm{h}(\x,\y,\z,\delta),\y)\|_2 - 1)^2},
\end{equation}
where $\partial_1 d(.,.)$ denotes the derivative with respect to the first argument of the critic. Note that, the derivative operator $\nabla d(.,.)$ in \Cref{eqn:full_GP} has been replace with $\partial_1 d(.,.)$ in \Cref{eqn:part_GP}.
\end{remark}

\subsection{Training}

To train the proposed cWGAN with Full-GP, we assume access to a dataset $\mathcal{S} = \{(\x^{(i)},\y^{(i)}) : 1 \leq i \leq N_s\}$ of paired samples from the true joint measure $\mu_{\X\Y}$. In practice, this can be constructed by first generating samples $\{\x^{(i)}: 1 \leq i \leq N_s\}$ sampled from a measure $\mu_\X$ (based on prior knowledge about the inferred field), and then generating the corresponding $\y_i$'s using the forward (noisy) model from \Cref{eqn:forward}. Alternatively, such paired data might be available from experiments.

Next, by replacing the expectations in the objective function \Cref{eqn:objective} by empirical averages, and choosing suitable architectures for the generator and critic (along with other hyper-parameters), the minmax problem \Cref{eqn:minmax2} is solved to find the optimal critic $d^*$ and generator $\g^*$. This involves using a gradient-based optimizer, such as Adam \cite{kingma2017adam}, to optimize the trainable parameters of the two networks. A computationally tractable strategy to solve the minmax problem is to use an iterative approach to train the networks. Typically, 
$N_\mathrm{max} \approx 4-10$
maximization steps are taken to update the trainable parameters of the critic, followed by a single minimization update step for the generator. The training is terminated after a certain number of epochs (loops over the training dataset) are completed. 

\subsection{Evaluation}\label{sec:evalution}

Consider the expectation  
of a continuous and bounded functional $q \in C_b(\Ro^\Nx)$ with respect to the true posterior conditional measure $\mu_{\X|\Y}(.|\y)$ 
\begin{equation}\label{eqn:cond_exp}
	k(\y;q) := \ex{\mu_{X|Y}}{q(\X)|\y} = \int_{\Ro^{\Nx}}q(\x)  \ud \mu_{\X|\Y}(\x|\y)
\end{equation}
for a given $\y \sim \mu_\Y$. For instance, choosing $q(\x) = x_i$ and $q(\x) = (x_i - \bar{x}_i)^2$ in \Cref{eqn:cond_exp} for $1 \leq i \leq \Nx$ leads to the evaluation of the component-wise mean $\bar{x}_i$ and variance $\sigma_{x,i}^2$ of each component of $\X$. The conditional expectation in \Cref{eqn:cond_exp} can be approximated using the trained generator $\g^*$ of the cWGAN using the following algorithm: Given $\hat{\y} \sim \mu_\Y$, $q \in C_b(\Ro^\Nx)$, sample size $K$ and $\sigma \geq 0$
\begin{enumerate}
	\item Draw $K$ random samples $\{\widetilde{\y}^{(i)}: 1 \leq i \leq K\}$ from $\mathcal{N}(\hat{\y},\sigma^2 \bm{I})$, i.e., the Gaussian measure centered at $\hat{\y}$ with covariance $\sigma^2 \bm{I}$.
	\item Generate $K$ samples $\{\z^{(i)}: 1 \leq i \leq K\}$ from $\mu_\Z$.
	\item Approximate the conditional expectation \eqref{eqn:cond_exp} using the empirical average
	\begin{equation}\label{eqn:app_cond_exp}
		k(\hat{y};q) \approx \frac{1}{K}\sum_{i=1}^{K} q\big(\g^*(\z^{(i)},\widetilde{\y}^{(i)})\big).
	\end{equation}
\end{enumerate}
The motivation behind using \Cref{eqn:app_cond_exp} to approximate the conditional expectations (and the choice of $\sigma$) is given by Theorem \ref{thm:errB} in Section \ref{sec:convergence} (see the discussion following its proof).

\section{Convergence analysis of the cWGAN with Full-GP}\label{sec:convergence}
In this section, we analyze the convergence of the cWGAN with Full-GP and show how the optimal generator can be used to approximate the posterior statistics from \Cref{eqn:cond_exp} in accordance to \Cref{eqn:app_cond_exp}. 
This proof, which forms the main result of this section, is contained in Theorem \ref{thm:errA} for $\sigma = 0$ and in Theorem \ref{thm:errB} for $\sigma >0$.
The proofs of these theorems rely on three distinct assumptions. The first (Assumption 4.1) is a statement of the convergence of a standard (not conditional) WGAN with increasing number of learnable parameters. This is also implicitly assumed in the original WGAN papers \cite{arjovsky2017,gulrajani2017improved}. The second assumption (Assumption 4.2) assumes that the true and learned conditional expectations are bounded. Finally, the third (Assumption 4.3) invokes the absolutely continuity of the marginal measure of measurement vector and ensuring the existence of a bounded marginal density.

Let $\mathcal{G}_n$ be the family of generator networks such that any network $\g \in \mathcal{G}_n$ has $n$ trainable parameters. 
We begin with our first assumption about the convergence of a standard WGAN with increasing \textit{capacity} of the networks.

\begin{assumption}[WGAN convergence]\label{ass:gan_conv}
	We assume that for every $n \in \mathbb{Z}_+$, we can find a minimizing generator $\g^*_n \in \mathcal{G}_n$ and a critic $d^*_n = d^*(\g_n) \in \lip$ such that $d^*_n$ solves the maximization problem in \Cref{eqn:minmax2}. In other words, by defining
	\[
	w_n := \loss(\g^*_n,d^*_n(\g^*_n)) =  W_1(\mu_{\X\Y},\mu^{\g^*_n}_{\X\Y}),
	\]
	for any other $\g_n \in \mathcal{G}_n$, we have
	\[
	W_1(\mu_{\X\Y},\mu^{\g_n}_{\X\Y}) = \loss(\g_n,d^*(\g_n)) \geq w_n.
	\]
	Further, the sequence of minimizing generators $\{\g^*_n\}$ leads to the following convergence
	\[
	\lim_{n \rightarrow \infty} w_n = 0.
	\]
\end{assumption}

Under Assumption \ref{ass:gan_conv}, the sequence of measures $\mu^{\g^*_n}_{\X\Y}$ will converge to $\mu_{\X\Y}$ in the $W_1$ metric, and hence weakly \cite{villani2008optimal}, i.e.,
\begin{equation}\label{eqn:weak_conv}
	\lim_{n \rightarrow \infty} \ex{\mu^{\g^*_n}_{\X\Y}}{\ell(\X,\Y)} =  \ex{\mu_{\X\Y}}{\ell(\X,\Y)} \quad \forall \ \ell \in C_b(\Ro^{\Nx} \times \Ro^{\Ny}).
\end{equation}

Next, we show how an appropriate choice of $\ell \in C_b(\Ro^{\Nx} \times \Ro^{\Ny})$ leads to approximating the conditional expectation \eqref{eqn:cond_exp}. Given a generator $\g$, we define the generated conditional expectations for any $q \in C_b(\Ro^{\Nx})$ as 
\begin{equation}\label{eqn:cond_exp_g}
	k^\g(\y;q) =\ex{\mu^\g_{X|Y}}{q(\X)|\y} = \int_{\Ro^{\Nx}}q(\x)  \ud \mu^\g_{\X|\Y}(\x|\y).
\end{equation}
We require the conditional expectations to be finite and computable.
\begin{assumption}[Bounded conditional expectation]\label{ass:bound_exp}
	For any $q \in C_b(\Ro^{\Nx})$, we assume that:
	\begin{enumerate}
		\item $\|k(\y;q)\|_{L^{\infty}(\Ro^\Ny)} < \infty$; and that
		\item there exists a positive integer $N_b$ (which may depend on $q$) such that 
        \begin{equation*}
            \|k^{\g^*_n}(\y;q)\|_{L^{\infty}(\Ro^\Ny)} \leq C_q < \infty
        \end{equation*}
        for all optimized generators $g^*_n$ with $n\geq N_b$. In other words, we assume a uniform bound (over $q$) for optimal generators beyond a certain network size.
	\end{enumerate} 
\end{assumption}
We also make the following assumption about the true marginal $\mu_\Y$ to simplify the exposition.
\begin{assumption}[Absolute continuity]\label{ass:abs_con}
	We assume that $\mu_\Y$ is absolutely continuous with respect to the Lebesgue measure on $\Ro^{\Ny}$, which ensures the existence of density $p_\Y$ satisfying
	\[
	\ud \mu_\Y(\y) = p_\Y(\y)\ud \y.
	\]
	Furthermore, we assume that $\|p_\Y\|_{L^{\infty}(\Ro^\Ny)} < \infty$.
\end{assumption}

Now, let $\hat{\y} \sim \mu_\Y$ be given for which $p_\Y(\hat{\y}) \neq 0$. Let $\mu_{\Y_\sigma} := \mathcal{N}(\hat{y},\sigma^2 \bm{I})$ be the Gaussian measure on $\Ro^{N_Y}$ with density 
\begin{equation}\label{eqn:gaussian}
	p_\sigma(\y) = \frac{1}{(2 \pi \sigma^2)^{n/2}}\exp\left( -\|\y - \hat{\y}\|^2_{L^{2}(\Ro^\Ny)}/ 2 \sigma^2 \right).
\end{equation}
The following standard result shows how the Dirac measure can be approximated using a Gaussian.
\begin{lemma}\label{lem:dirac_app}
	Let $f \in L^\infty(\Ro^{\Ny})$ and $\mu_{\Y_\sigma} := \mathcal{N}(\hat{y},\sigma^2 \bm{I})$. Then
	\begin{equation}\label{eqn:dirac_app}
		\lim_{\sigma \rightarrow 0} \ex{\mu_{\Y_\sigma}}{f(\Y)} = f(\hat{\y}).
	\end{equation}
	
\end{lemma}
\begin{proof}
	Note that
	\[
	\ex{\mu_{\Y_\sigma}}{f(\Y)} = \int_{\Ro^{\Ny}} f(\y) \frac{1}{(2 \pi \sigma^2)^{n/2}}\exp\left( -\|\y - \hat{\y}\|^2_{L^{2}(\Ro^\Ny)}/2\sigma^2 \right) \ud \y.
	\]
	Applying a change of variables $\bm{s} = (\y - \hat{\y})/\sigma$, we get
	\[
	\ex{\mu_{\Y_\sigma}}{f(\Y)} = \int_{\Ro^{\Ny}} f(\sigma\bm{s} + \hat{\y}) \frac{1}{(2 \pi)^{n/2}}\exp\left( -\|\bm{s}\|^2_{L^{2}(\Ro^\Ny)} /2 \right) \ud \bm{s}.
	\]
	Since the integrand converges a.e. to $f(\hat{\y}) \frac{1}{(2 \pi)^{n/2}}\exp\left( -\|\bm{s}\|^2_{L^{2}(\Ro^\Ny)}/2 \right)$ as $\sigma \rightarrow 0$ and 
	\[
	\int_{\Ro^{\Ny}} \frac{1}{(2 \pi)^{n/2}}\exp\left( -\|\bm{s}\|^2_{L^{2}(\Ro^\Ny)} /2\right) \ud \bm{s} = 1
	\]
	we get \eqref{eqn:dirac_app} by an application of the dominated convergence theorem.
\end{proof}

For the rest of the theoretic discussion, we omit the superscript $*$ and it is understood that $\g_n$ denotes the optimal generator in accordance to Assumption \ref{ass:gan_conv}.
We are now ready to state and prove the main results about approximating \Cref{eqn:cond_exp}.

\begin{theorem}\label{thm:errA}
	Let $\hat{\y} \sim \mu_\Y$ be given for which $p_\Y(\hat{\y}) \neq 0$. Let the Assumptions \ref{ass:gan_conv}, \ref{ass:bound_exp} and \ref{ass:abs_con} hold. Then given $\epsilon > 0$ and $q \in C_b(\Ro^\Nx)$, there exists a positive integer $\widetilde{N} := \widetilde{N}(\hat{\y},q,\epsilon) \geq N_b$ such that
	\begin{equation}\label{eqn:err_estimate_A}
            \left| k(\hat{\y};q) - k^{\g_n}(\hat{\y};q) \right| < \epsilon \quad \forall \ n \geq \widetilde{N}
	\end{equation}
	Note that $N_b$ is as defined in Assumption \ref{ass:bound_exp}.
\end{theorem}

\begin{proof}
	Consider the Gaussian measure $\mu_{\Y_\sigma}$ whose density $p_\sigma$ is given by \Cref{eqn:gaussian}. Let $\ell_\sigma(\x,\y) = q(x)p_\sigma(\y)$, which clearly belongs to $C_b(\Ro^{\Nx} \times \Ro^{\Ny})$. By Assumption \ref{ass:abs_con}, we have
	\begin{eqnarray}\label{eqn:estimate1}
		\ex{\mu_{\X\Y}}{\ell_\sigma(\X,\Y)} &=& \int_{\Ro^{\Nx} \times \Ro^{\Ny}} q(\x)p_\sigma(\y) \ud \mu_{\X\Y}(\x,\y) \notag \\
		&=& \int_{\Ro^{\Ny} } \int_{\Ro^{\Nx} }q(\x) p_\sigma(\y) \ud \mu_{\X|\Y}(\x|\y) \ud \mu_{\Y}(\y)\notag \\
		&=& \int_{\Ro^{\Ny} }  k(\y;q) p_\sigma(\y) \ud \mu_{\Y}(\y)\notag \\
		&=& \int_{\Ro^{\Ny} }  k(\y;q) p_\sigma(\y) p_\Y(\y) \ud \y \notag \\
		&=& \ex{\mu_{\Y_\sigma}}{k(\Y;q) p_\Y(\Y) }.
	\end{eqnarray}
	Then using Assumptions \ref{ass:bound_exp}, \ref{ass:abs_con}, and Lemma \ref{lem:dirac_app} in \Cref{eqn:estimate1} leads to
	\begin{equation}\label{eqn:estimate2}
		\lim_{\sigma \rightarrow 0} \ex{\mu_{\X\Y}}{\ell_\sigma(\X,\Y)} = k(\hat{\y};q) p_\Y(\hat{\y}).
	\end{equation}
	Thus, there exists $\widetilde{\sigma}_1 = \widetilde{\sigma}_1(\epsilon,q,\hat{\y})$ such that
	\begin{equation}\label{eqn:eps_est1}
		\left| \frac{1}{p_\Y(\hat{\y})} \ex{\mu_{\X\Y}}{\ell_\sigma(\X,\Y)} - k(\hat{\y};q) \right| < \frac{\epsilon}{3} \quad \forall \ \sigma \leq \widetilde{\sigma}_1.
	\end{equation}
	
	\noindent Similarly, we can show that 
	\begin{eqnarray*}
		\ex{\mu^{\g_n}_{\X\Y}}{\ell_\sigma(\X,\Y)} &=& \ex{\mu_{\Y_\sigma}}{k^{\g_n}(\Y;q) p_\Y(\y) },
	\end{eqnarray*}
	and for any optimal generator $\g_n$ with $n \geq N_b$
	\begin{equation}\label{eqn:estimate3}
		\lim_{\sigma \rightarrow 0} \ex{\mu^{\g_n}_{\X\Y}}{\ell_\sigma(\X,\Y)} = k^{\g_n}(\hat{\y};q) p_\Y(\hat{\y}).
	\end{equation}
	In other words there exists $\widetilde{\sigma}_2 = \widetilde{\sigma}_2(\epsilon,q,\hat{\y})$ such that for $n \geq N_b$
	\begin{equation}\label{eqn:eps_est2}
		\left| \frac{1}{p_\Y(\hat{\y})} \ex{\mu^{\g_n}_{\X\Y}}{\ell_\sigma(\X,\Y)} - k^{\g_n}(\hat{\y};q) \right| < \frac{\epsilon}{3} \quad \forall \ \sigma \leq \widetilde{\sigma}_2.
	\end{equation}
	Note that $\widetilde{\sigma}_2$ will be independent on ${\g_n}$ because of the uniform bound assumption in Assumption \ref{ass:bound_exp} for $n \geq N_b$.
	
	We note that by choosing $\sigma = \widetilde{\sigma} := \min(\widetilde{\sigma}_1,\widetilde{\sigma}_2)$ both inequalities \labelcref{eqn:eps_est1,eqn:eps_est2} hold simultaneously. For this $\widetilde{\sigma}$ and the corresponding $\ell_{\widetilde{\sigma}}$, we have the weak convergence estimate \Cref{eqn:weak_conv} due to Assumption \ref{ass:gan_conv}. Thus, there exists a positive integer $\widetilde{N}_1 := \widetilde{N}_1(\widetilde{\sigma}) = \widetilde{N}_1(\hat{\y},q,\epsilon)$ such that
	\begin{equation}\label{eqn:eps_est3}
		|\ex{\mu^{\g_n}_{\X\Y}}{\ell_{\widetilde{\sigma}}(\X,\Y)} -  \ex{\mu_{\X\Y}}{\ell_{\widetilde{\sigma}}(\X,\Y)} | < \frac{\epsilon p_\Y(\hat{y})}{3} \quad \forall \ n \geq \widetilde{N}_1.
	\end{equation}
   Further, by setting 
	$\widetilde{N}:=\max(N_b,\widetilde{N}_1)$, we have using inequalities \labelcref{eqn:eps_est1,eqn:eps_est2,eqn:eps_est3} that for all $n\geq \widetilde{N}$
 \begin{eqnarray*}
		\left| k(\hat{\y};q) - k^{\g_n}(\hat{\y};q) \right| &\leq& \left| k(\hat{\y};q) - \frac{1}{p_\Y(\hat{\y})} \ex{\mu_{\X\Y}}{\ell_{\widetilde{\sigma}}(\X,\Y)} \right| \\
		&& + \frac{1}{p_\Y(\hat{\y})}\left|\ex{\mu_{\X\Y}}{\ell_{\widetilde{\sigma}}(\X,\Y)} - \ex{\mu^{\g_n}_{\X\Y}}{\ell_{\widetilde{\sigma}}(\X,\Y)} \right| \notag \\
		&& + \left| \frac{1}{p_\Y(\hat{\y})} \ex{\mu^{\g_n}_{\X\Y}}{\ell_{\widetilde{\sigma}}(\X,\Y)} - k^{\g_n}(\hat{\y};q) \right| \notag \\
		&<& \frac{\epsilon}{3} + \frac{1}{p_\Y(\hat{\y})} \frac{\epsilon p_\Y(\hat{\y}) }{3} + \frac{\epsilon}{3} = \epsilon.
	\end{eqnarray*}
\end{proof}

We construct a sequence of joint measures with elements
\begin{equation}\label{eqn:delta_joint}
	\hat{\mu}^{\g_n}_{\X\Y_\sigma}(\x,\y) = \mu^{\g_n}_{\X|\Y}(\x|\y) \mu_{\Y_\sigma}(\y),   
\end{equation}
which is different from $\mu_{\X\Y}$ or $\mu^{\g_n}_{\X\Y}$. We now prove a robustness result which shows that we can obtain error estimates similar to Theorem \ref{thm:errA} even if we inject a controlled amount of noise into the given measurement $\y$ in accordance to $\mu_{\Y_\sigma}$.  

\begin{theorem}\label{thm:errB}
	Let $\hat{\y} \sim \mu_\Y$ be given for which $p_\Y(\hat{\y}) \neq 0$. Let the Assumptions \ref{ass:gan_conv}, \ref{ass:bound_exp} and \ref{ass:abs_con} hold. Then given $\epsilon > 0$ and $q \in C_b(\Ro^\Nx)$, there exists a positive real number $\bar{\sigma} := \bar{\sigma}(\hat{y},q,\epsilon)$ and a positive integer with $\bar{N} := \bar{N}(\bar{\sigma}) \geq N_b$ such that
 \begin{equation}\label{eqn:err_estimate_B}
        \left| k(\hat{\y};q) - \ex{\hat{\mu}^{\g_n}_{\X\Y_{\bar{\sigma}}}}{q(\X)} \right| < \epsilon \quad \forall \ n \geq \bar{N},
	\end{equation}
	where $\hat{\mu}^{\g_n}_{\X\Y_{\bar{\sigma}}}$ is defined according to \Cref{eqn:delta_joint}.  
\end{theorem}
\begin{proof}
	Following the proof of Theorem \ref{thm:errA}, we can find a $\bar{\sigma}_1 = \bar{\sigma}_1(\epsilon,q,\hat{y})$ and a $\bar{\sigma}_2 = \bar{\sigma}_2(\epsilon,q,\hat{y})$ such that
	\begin{equation}\label{eqn:eps_est5}
		\left| \frac{1}{p_\Y(\hat{\y})} \ex{\mu_{\X\Y}}{\ell_\sigma(\X,\Y)} - k(\hat{\y};q) \right| < \frac{\epsilon}{4} \quad \forall \ \sigma \leq \bar{\sigma}_1.
	\end{equation}
	and
	\begin{equation}\label{eqn:eps_est6}
		\left| \frac{1}{p_\Y(\hat{\y})} \ex{\mu^{\g_n}_{\X\Y}}{\ell_\sigma(\X,\Y)} - k^{\g_n}(\hat{\y};q) \right| < \frac{\epsilon}{4} \quad \forall \ \sigma \leq \bar{\sigma}_2 \ \ \text{and} \ \ n \geq N_b.
	\end{equation}
	
	We also have
	\begin{eqnarray*}
		\ex{\hat{\mu}^{\g_n}_{\X\Y_\sigma}}{q(\X)}&=& \int_{\Ro^{\Nx} \times \Ro^{\Ny}} q(\x) \ud \hat{\mu}^{\g}_{\X\Y_\sigma}(\x,\y)\notag \\
		&=& \int_{\Ro^{\Ny} } \int_{\Ro^{\Nx} }q(\x) \ud \mu^{\g_n}_{\X|\Y}(\x|\y) \ud \mu_{\Y_\sigma}(\y) \notag \\
		&=& \int_{\Ro^{\Ny} } k^{\g_n}(\y;q) \ud \mu_{\Y_\sigma} \notag \\
		&=& \ex{\mu_{\Y_\sigma}}{k^{\g_n}(\Y;q)}.
	\end{eqnarray*}
	By Assumption \ref{ass:bound_exp} and Lemma \ref{lem:dirac_app}, we have for $n\geq N_b$
	\begin{equation}
		\lim_{\sigma \rightarrow 0} \ex{\hat{\mu}^{\g_n}_{\X\Y}}{q(\X)} = k^{\g_n}(\hat{\y};q).
	\end{equation}
	Thus there exists $\bar{\sigma}_3 := \bar{\sigma}_3(\epsilon,q,\hat{\y})$ such that for $n\geq N_b$
	\begin{equation}\label{eqn:eps_est7}
		| \ex{\hat{\mu}^{\g_n}_{\X\Y}}{q(\X)} - k^{\g_n}(\hat{\y};q) | < \frac{\epsilon}{4} \quad \forall \ \sigma \leq \bar{\sigma}_3.
	\end{equation}
	
	We set $\sigma = \bar{\sigma} := \min(\bar{\sigma}_1,\bar{\sigma}_2,\bar{\sigma}_3)$ and consider the corresponding $\ell_{\bar{\sigma}}$. Thus, by \Cref{eqn:weak_conv} there exists a positive integer $\bar{N}_1 = \bar{N}_1(\bar{\sigma})$ such that
	\begin{equation}\label{eqn:eps_est8}
		|\ex{\mu^{\g_n}_{\X\Y}}{\ell_{\bar{\sigma}}(\X,\Y)} -  \ex{\mu_{\X\Y}}{\ell_{\bar{\sigma}}(\X,\Y)} | < \frac{\epsilon p_\Y(\hat{y})}{4} \quad \forall \ n \geq \bar{N}_1.
	\end{equation}
	
	By setting 
	$\bar{N} = \max(\bar{N}_1,N_b)$ we have using inequalities \labelcref{eqn:eps_est5,eqn:eps_est6,eqn:eps_est7,eqn:eps_est8} that for all $n \geq \bar{N}$
 \begin{eqnarray*}
		\left| k(\hat{\y};q) - \ex{\hat{\mu}^{\g_n}_{\X\Y_{\bar{\sigma}}}}{q(\X)} \right| &\leq& \left| k(\hat{\y};q) - \frac{1}{p_\Y(\hat{\y})} \ex{\mu_{\X\Y}}{\ell_{\bar{\sigma}}(\X,\Y)} \right| \\
		&& + \frac{1}{p_\Y(\hat{\y})}\left|\ex{\mu_{\X\Y}}{\ell_{\bar{\sigma}}(\X,\Y)} - \ex{\mu^{\g_n}_{\X\Y}}{\ell_{\bar{\sigma}}(\X,\Y)} \right| \notag \\
		&& + \left| \frac{1}{p_\Y(\hat{\y})} \ex{\mu^{\g_n}_{\X\Y}}{\ell_{\bar{\sigma}}(\X,\Y)} - k^{\g_n}(\hat{\y};q) \right| \notag \\
		&& + \left| k^{\g_n}(\hat{\y};q) - \ex{\hat{\mu}^{\g_n}_{\X\Y_{\bar{\sigma}}}}{q(\X)} \right| \\
		&<& \frac{\epsilon}{4} + \frac{1}{p_\Y(\hat{\y})} \frac{\epsilon p_\Y(\hat{\y}) }{4} + \frac{\epsilon}{4} + \frac{\epsilon}{4} = \epsilon.
	\end{eqnarray*}
\end{proof}

Note that
\begin{eqnarray}\label{eqn:cond_exp_x}
	\ex{\hat{\mu}^{\g_n}_{\X\Y_\sigma}}{q(\X)}&=& \int_{\Ro^{\Nx} \times \Ro^{\Ny}} q(\x) \ud \hat{\mu}^{\g_n}_{\X\Y_\sigma}(\x,\y)\notag \\
	&=& \int_{\Ro^{\Ny} } \int_{\Ro^{\Nx} }q(\x) \ud \mu^{\g_n}_{\X|\Y}(\x|\y) \ud \mu_{\Y_\sigma}(\y) \notag \\
	&=& \int_{\Ro^{\Ny} } \int_{\Ro^{\Nz} }q(\g_n(\z,\y)) \ud \mu_\Z(\z) \mu_{\Y_\sigma}(\y).
\end{eqnarray}
Taking the Monte Carlo approximation of \Cref{eqn:cond_exp_x} leads to the approximation formula \Cref{eqn:app_cond_exp} in our algorithm, with the choice of $\sigma$ motivated by the proof of Theorem \ref{thm:errB}. In particular, to get a reasonable approximation of the conditional expectation, we need to pick a $\sigma = \bar{\sigma}$ small enough to ensure that inequalities \labelcref{eqn:eps_est5,eqn:eps_est6,eqn:eps_est7} are satisfied and a generator large (expressive) enough to ensure the inequality \labelcref{eqn:eps_est8} is satisfied.

\subsection{Comparison with the Partial-GP approach of \cite{adler2018deep} }
We briefly discuss the theoretical inferences of the Partial-GP approach considered for cWGANs in \cite{adler2018deep}, and how it differs from the Full-GP approach. 
Let us define by $\lipx$ the space of 1-Lipschitz function in $\Ro^{\Nx}$, and by $\lipxe$ the space of functions measurable in $\Ro^{\Nx} \times \Ro^{\Ny}$ that are 1-Lipschitz in $\Ro^{\Nx}$ for all $\y$. Clearly, $\lip \subset \lipxe$ and $d \in \lipxe$ implies $d(.,y) \in \lipx$ for any $y \in \dby$. The authors of \cite{adler2018deep} solve the maximization problem \eqref{eqn:max_with_GP} with \eqref{eqn:part_GP} but over the larger space 
$\lipxe$
\begin{eqnarray}\label{eqn:max_with_part_GP}
	d^*(\g) &=& \argmax{d \in \lipxe} \left[\loss(\g,d) - \lambda \mathcal{G}\mathcal{P} \right] \notag \\
	&=& \argmax{d \in \lipxe} \big\{  \ex{\mu_{\X\Y}}{d(\X,\Y)} - \ex{\mu^{\g}_{\X\Y}}{d(\X,\Y)} \big\} \notag \\
	&=& \argmax{d \in \lipxe} \ex{\mu_\Y}{ \ex{\mu_{\X|\Y}}{d(\X, \Y)|\Y}   - \ex{\mu^{\g}_{\X|\Y}}{d(\X, \Y)|\Y}}.
\end{eqnarray}
The authors then proceed to 
prove that 
the $\argmax{d \in \lipx}$ and $\mathbb{E}_{\mu_\Y}$ operators in \Cref{eqn:max_with_part_GP} 
may be interchanged 
under the following technical assumption
\begin{assumption}\label{ass:adler}
Given a generator $\g$ and $\epsilon > 0$,
there exists a measurable mapping $D:(\x,\y) \mapsto \hat{d}_\y(x)$ such that for all $\y \in \dby$ it holds that $\hat{d}_\y \in \lipx$  and 
\[
\ex{\mu_{\X|\Y}}{\hat{d}_\y(\X)|\y}   - \ex{\mu^{\g}_{\X|\Y}}{\hat{d}_\y(\X)|\y} > \max\limits_{d \in \lipxe} \left\{\ex{\mu_{\X|\Y}}{d(\X, \Y)|\y}   - \ex{\mu^{\g}_{\X|\Y}}{d(\X, \Y)|\y} \right\} - \epsilon.
\]
\end{assumption}
We note that the Assumption \ref{ass:adler} may not be easy to verify, nor is it easy to interpret the mapping $D$.

Thereafter, by an application of the Kantorovich–Rubinstein duality argument, they obtain \begin{eqnarray}\label{eqn:max_with_part_GP_cont}
	d^*(\g) &=&  \ex{\mu_\Y}{ \argmax{d \in \lipxe} \left\{\ex{\mu_{\X|\Y}}{d(\X, \Y)|\Y}   - \ex{\mu^{\g}_{\X|\Y}}{d(\X, \Y)|\Y}\right\}} \notag \\
	&=& \ex{\mu_\Y}{W_1(\mu_{\X|\Y}, \mu^{\g}_{\X|\Y})},
\end{eqnarray}
which is the `\textit{expected}' (mean) $W_1$ metric with respect to $\mu_\Y$ between the true conditional measure and the learned conditional measure. Thus, finding the optimal generator $\g^*$ corresponds to minimizing this `\textit{expected}' (mean) $W_1$ metric between the conditionals with respect to the true measure $\mu_{\Y}$. However, finding a sequence of generators $\g_n^*$ such that 
\begin{equation}\label{eqn:adler_conv}
	\lim_{n\rightarrow \infty} \ex{\mu_\Y}{W_1(\mu_{\X|\Y}, \mu^{\g_n^*}_{\X|\Y})} = 0
\end{equation}
does not guarantee weak convergence of the measures $\mu^{\g_n^*}_{\X|\Y}$ to $\mu_{\X|\Y}$. 
In fact, by the properties of convergence in $L^1(\mu_\Y)$, \cref{eqn:adler_conv} would only guarantee the existence of a sub-sequence of measures $\mu^{\g_{n_k}}_{\X|\Y}$ such that $\lim_{k\rightarrow \infty} W_1(\mu_{\X|\Y}, \mu^{\g_{n_k}^*}_{\X|\Y}) = 0$, which in turn would lead to the weak-convergence result
\[
\lim_{k \rightarrow \infty} \ex{\mu^{\g^*_{n_k}}_{\X|\Y}}{q(\X)|\hat{\y}} =  \ex{\mu_{\X|\Y}}{q(\X)|\hat{\y}} \quad \forall \ q \in C_b(\Ro^{\Nx}).
\]

In contrast to Partial-GP, Full-GP is able to provide an alternate route to prove weak convergence. The primary difference stems from formulating the maximization \Cref{eqn:max_with_GP} over 
$\lip \subset \lipxe$. 
First, this allows us to consider the weak convergence of the learnt joint  measures $\mu^{\g_n^*}_{\X\Y}$ to the true joint measure $\mu_{\X\Y}$ without requiring a
a technical assumption like Assumption \ref{ass:adler} or a 
sub-sequential argument. Then, by using the fact that a Gaussian measure closely resembles a Dirac measure in the asymptotic limit of diminishing variance (Lemma \ref{lem:dirac_app}), we are able to derive error estimates for approximating the conditional expectation in \Cref{eqn:cond_exp} by sampling using the trained generator (Theorems \ref{thm:errA} and \ref{thm:errB}).

\section{Numerical experiments}\label{sec:results}

In this section, several numerical experiments are presented to analyze the performance of the proposed cWGANs with Full-GP. The experiments include: one-dimensional problems where the target measure is known; an inverse heat conduction problem where the reference statistics of the target measure are available \cite{ray2022}; and a complex inverse problem in elastography where the target measure is unknown. For the first two problems, where the target statistics are available and we have a reliable reference to estimate error, we use both the Partial-GP  \cite{adler2018deep} (see \Cref{eqn:part_GP}), and Full-GP (see \Cref{eqn:full_GP}) approaches to perform a  comparison of the two methods.

\subsection{Illustrative examples}\label{sec:illus_ex}
We first analyze the performance of the proposed methodology using one-dimensional random variables 
$X$ and $Y$ (the paired random variable, 
$(X, Y)$, being two-dimensional) where the target probability measures are available. These are defined as follows:

\begin{align}
    \text{Tanh + $\Gamma$}:& \quad x = \operatorname{tanh}(y) + \gamma \text{ where } \gamma \sim \Gamma(1, 0.3) \text{ and } \ y \sim U(-2,2) \label{eqn:tanhGamma} \\
    \text{Bimodal}:& \quad x = (y+w)^{1/3} \text{ where } y \sim \mathcal{N}(0, 1) \text{ and } w \sim \mathcal{N}(0, 1) \label{eqn:bimodal} \\
    \begin{split}
        \text{Swissroll}:& \quad x = 0.1t\sin(t)+0.1w, \ y = 0.1t\cos(t)+0.1v, \ t=3\pi/2(1+2h), \\
        &\quad \text{where } h \sim U(0,1), \ w \sim \mathcal{N}(0,1) \text{ and } \ v \sim \mathcal{N}(0,1)
    \end{split} \label{eqn:swissroll}
\end{align}
Here $\mathcal{N},U$ and $\Gamma$ refer to the Normal, Uniform, and Gamma distributions, respectively. Column 1 of \Cref{fig:toyprob_jointhist} shows the true joint distribution $\mu_{XY}$ corresponding to \Cref{eqn:tanhGamma,eqn:bimodal,eqn:swissroll}.
We consider the problem of learning the distribution of $x$ conditioned on $y$. 
For all three problems, we construct a training dataset that is composed of 2,000 samples of $x$ and $y$. For this set of problems, we compare the performance of cWGANs with Partial-GP and Full-GP. For the cWGAN with Full-GP we consider two cases: $\sigma=0$ (\ie $k(\hat{y};q)$ is estimated as $\sum_{i=1}^{K} q(\g^*(z_i,\hat{y}))/ K)$ for a given $y = \hat{y}$), and an optimally chosen $\sigma$ (\ie $k(\hat{y};q)$ estimated using \Cref{eqn:app_cond_exp}). The optimal value of $\sigma$ is chosen such that the batch Wasserstein-1 distance, $\tilde{W}_1$, between $\mu^{\g}_{X|Y}(\cdot|\hat{y})$ and $\mu_{X|Y}(\cdot|\hat{y})$ attains a minimum; we estimate the $\tilde{W}_1$ distance using the open source POT package \cite{flamary2021pot} and 10\textsuperscript{4} realizations. Moreover, $\sigma$ is scaled by $\sigma_y$ such that $\sigma/\sigma_y=\sigma^*$, where $\sigma_y$ is the standard deviation of the random variable $Y$. In practice, it is easier to work with $\sigma^*$ rather than $\sigma$ since the former is dimensionless and a scaled quantity and can be compared across different examples.

In the cWGANs, for both Full-GP and Partial-GP, we use multilayer perceptrons (MLPs) for the critic and generator networks and perform a hyper-parameter search (including for the gradient penalty weight $\lambda$) to determine the optimal configuration for each method. We provide further description of the problem set-up and additional details necessary to reproduce the results in \Cref{app:archs_2Dexp}.

\begin{figure}[H]
\centering
\includegraphics[width=16.5cm]{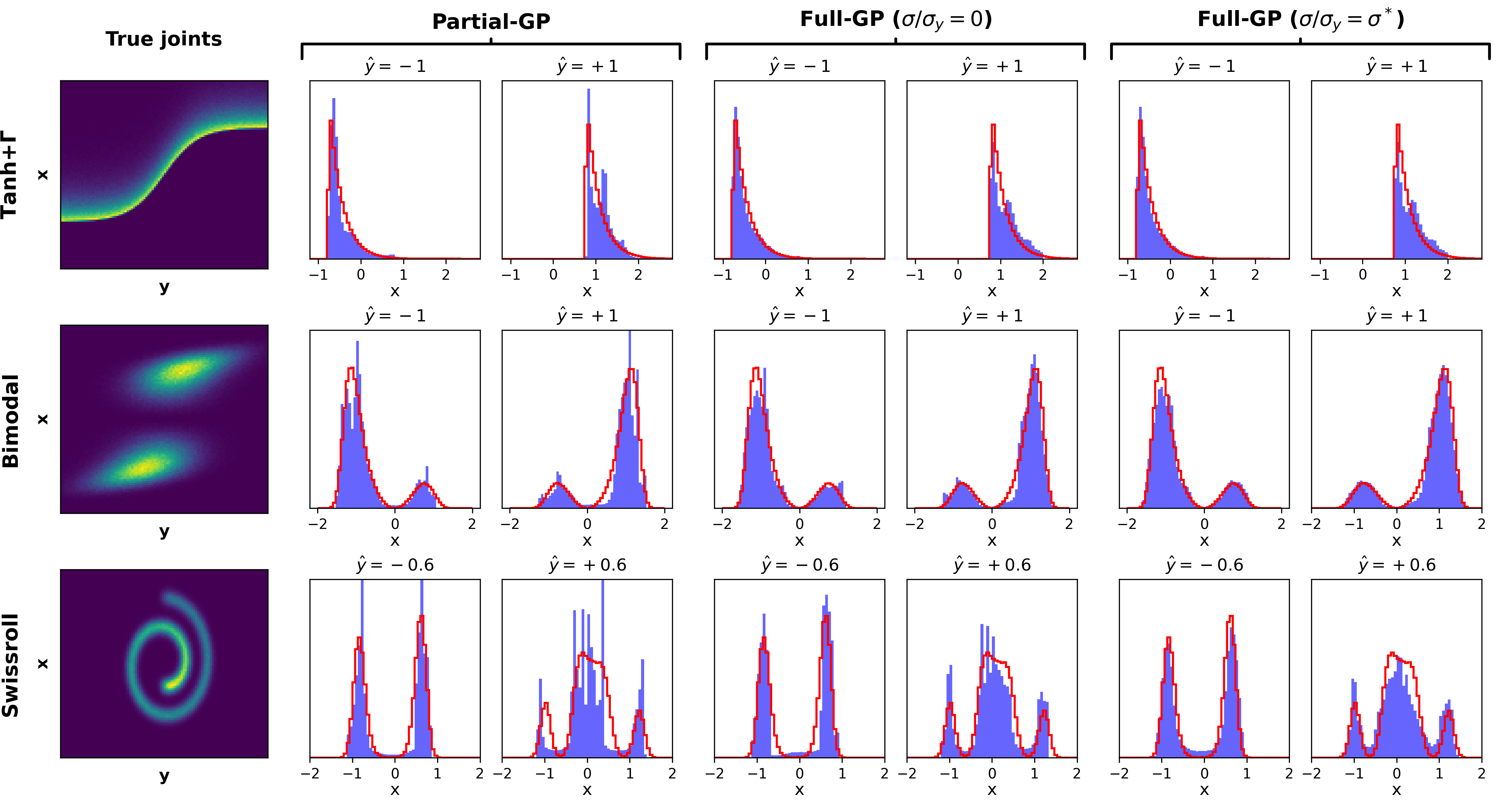}
\caption{The leftmost column shows histograms generated from 10\textsuperscript{6} samples from the true joint measure $\mu_{XY}$. The other columns show histograms of 10\textsuperscript{5} samples generated from the generated conditional measures $\mu^g_{X|Y}(\cdot|\hat{y})$ and, in red, the outline of histograms of samples from the true conditional measure $\mu_{X|Y} (\cdot|\hat{y})$}
\label{fig:toyprob_jointhist}
\end{figure}

\Cref{fig:toyprob_jointhist} shows the results when a cWGAN with Full-GP and Partial-GP attempts to learn the conditional measure $\mu_{X|Y}(\cdot|\hat{y})$ for the problems defined in \Cref{eqn:tanhGamma,eqn:bimodal,eqn:swissroll} corresponding to two separate values of $\hat{y}$.
Columns 2-7 show histograms of the samples drawn from $\mu^g_{X|Y}(\cdot|\hat{y})$ and, in red, the outline of the histogram of samples drawn from the true conditional distribution $\mu_{X|Y}(\cdot|\hat{y})$.  Columns 2 and 3 correspond to the cWGAN with Partial-GP, columns 4 and 5 correspond to Full-GP with $\sigma=0$, and columns 6 and 7 correspond to Full-GP with optimally chosen $\sigma = \sigma^* \sigma_y$. Although the difference between the histograms for the  distribution of $x$ conditioned on $y$ obtained using the various approaches is minor, we observe that in all three cases the Full GP approach yields histograms of the conditional distribution that are closer to their corresponding reference. The histograms obtained using cWGANs with Full-GP where $\sigma/\sigma_y=\sigma^*$ appear to be the best match to the reference. This improvement in performance may be attributed to the smoothing that occurs as a consequence of adding noise to the input $y$. Also note that, due to the limited quantity of training data (2,000 training samples), it is not possible to approximate the target measure too well. 

We quantify the performance of each method using the following two metrics: 
the batch Wasserstein-1 distance $\tilde{W}_1$, which is the same metric that is employed to choose $\sigma^*$, and the relative $L_2$ error between samples from the target and generated random variables. We define relative $L_2$ error as follows:
\begin{equation*}
   \text{Relative $L_2$ error} =  \int_{X\times Y} (\hat{p}_{XY}-\hat{p}^{g}_{XY})^2 \ud x \ud y \;\;\bigg/ \int_{X\times Y} (\hat{p}_{XY})^2 \ud x \ud y,
\end{equation*} 
where $\hat{p}_{XY}$ and $\hat{p}^g_{XY}$ denote the true and generated joint distributions, respectively. We estimate the relative $L_2$ errors using 10\textsuperscript{6} realizations.

Table \ref{table:toyproblem_errors} shows the minimum error achieved from four independent runs for each method and joint measure. For all three problems, cWGANs with Full-GP outperform cWGANs with Partial-GP on both metrics. Another observation that can be made from \Cref{table:toyproblem_errors} is that the joint distribution induced by cWGANs with Full-GP when $\sigma \neq 0$ have smaller $\tilde{W}_1$ distances as compared to the case where $\sigma = 0$, which is expected since this is the same metric that we use to choose $\sigma$, but the influence of $\sigma$ on the relative $L_2$ error is negligible.

\begin{table}[H]

  \caption{Error metrics between the true joint measure and the approximation obtained using different cWGANs for the illustrative examples}
  \label{table:toyproblem_errors}
  \centering
  \begin{tabular}{lcccccc}
    \toprule
     \multirow{2}{*}{Example}& \multicolumn{2}{c}{\textbf{Partial-GP}}  & \multicolumn{2}{c}{\textbf{Full-GP ($\sigma/\sigma_y=0$)}} & \multicolumn{2}{c}{\textbf{Full-GP ($\sigma/\sigma_y=\sigma^*$)}}  \\
     \cline{2-7}
     & $\tilde{W}_1$ dist. & rel. L2 error & $\tilde{W}_1$ dist. & rel. L2 error & $\tilde{W}_1$ dist. & rel. L2 error \\
    \midrule
    Tanh+$\Gamma$ & 0.0429 & 0.192 & 0.0427 & 0.131 & 0.0390 & 0.134\\
    Bimodal & 0.0509 & 0.201 & 0.0509 & 0.161 & 0.0471 & 0.163\\
    Swissroll & 0.0477 & 0.360 & 0.0483 & 0.284 & 0.0399 & 0.281\\
    \bottomrule
  \end{tabular}
\end{table}

\Cref{fig:toyprob_truevs_estW1} shows the variation in the $\tilde{W}_1$ distance between $\mu_{X|Y}(\cdot|\hat{y})$ and $\mu^g_{X|Y}(\cdot|\hat{y})$ (for the same instances of $\hat{y}$ as shown in Figure \ref{fig:toyprob_jointhist}) as the non-dimensional quantity $\sigma/\sigma_y$ is varied. The blue curve corresponds to Full-GP with $\sigma = 0$ and the red curve corresponds to Partial-GP. \Cref{fig:toyprob_truevs_estW1} shows that the change in the $\tilde{W}_1$ distance is negligible for $\sigma/\sigma_y < 0.1$, which may explain the limited influence of $\sigma$. Therefore, in a practical problem where the true condition density is not known, $\sigma$ need not be optimized as long as a small value is chosen. Further, we observe that the $\tilde{W}_1$ distance between $\mu_{X|Y}^g(\cdot|\hat{y})$ and $\mu_{X|Y}(\cdot|\hat{y})$ using the Partial-GP method is greater or similar to Full-GP for all problems at both instances of $\hat{y}$ considered here.

\begin{figure}[H]
\centering
\includegraphics[width=16.5cm]{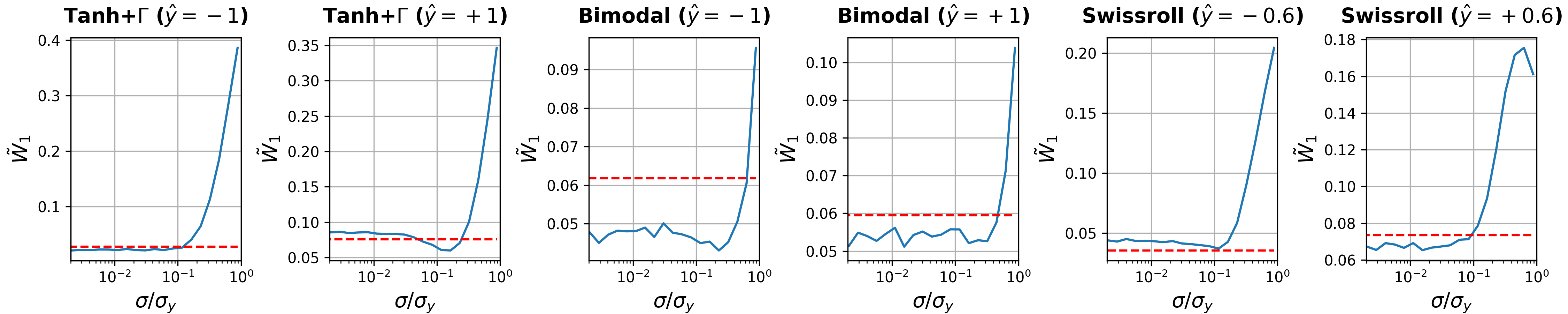}
\caption{Batch Wasserstein-1 distance, $\tilde{W}_1$, vs. $\sigma$ for cWGANs models with Full-GP. 
For reference, the $\tilde{W}_1$ distance for cWGANs with Partial-GP is shown using a dashed red line}
\label{fig:toyprob_truevs_estW1}
\end{figure}

\subsection{Inverse heat conduction} \label{sec:inv_heat_eq}
We consider the following two-dimensional heat conduction problem 
\begin{alignat}{2}
    \frac{\partial u(\s,t)}{\partial t} - \kappa \Delta u (\s,t)) &= 0, \qquad 
     &&\forall \ (\s,t) \in [0,2\pi]^2 \times (0, T)  \label{eqn:pde_t_heat} \\
    u (\s, 0) &= u_0(\s),  \qquad 
     &&\forall \ \s \in [0,2\pi]^2 \label{eqn:ic_t_heat} \\
    u (\s, t) &= 0, \qquad
     &&\forall \ (\s,t) \in \partial [0,2\pi]^2 \times (0, T) \label{eqn:bc_t_heat}
\end{alignat}
where $u$ denotes the temperature field, $\kappa$ denotes the conductivity field (which is set to 0.64), and $u_0$ denotes the initial temperature field. Given a noisy temperature field at final time $T=1$, we wish to infer the initial condition $u_0$. We remark that this inverse problem is very ill-posed problem due to the significant loss of information via diffusion \cite{engl2000}. 

For the Bayesian framework, we assume the following prior on the initial data
\begin{equation}\label{eqn:rect_prior}
u_0(\s) = \begin{cases}
2 +  2\frac{(s_1 - \xi_1)}{(\xi_3-\xi_1)}&\quad \text{if } s_1 \in [\xi_1,\xi_3], \ s_2 \in[\xi_4,\xi_2],\\
0 & \quad \text{otherwise}, 
\end{cases}
\end{equation}
with $\xi_1,\xi_4$ sampled from the uniform distribution on $[0.2,0.4]$, and $\xi_2,\xi_3$ sampled uniformly from $[0.6,0.8]$. The prior distribution defined by \Cref{eqn:rect_prior} corresponds to a linearly varying temperature field in a rectangular region of random shape/size with a zero background outside the rectangle. For the present inverse problem, we represent the discrete initial temperature field (to be inferred) on a $28 \times 28$ Cartesian grid by the random variable $\X$. The final temperature field is recovered by solving Eqs.~\eqref{eqn:pde_t_heat}, \eqref{eqn:ic_t_heat} and \eqref{eqn:bc_t_heat} using a central-space-backward-time finite difference scheme on the same gird. We then add noise sampled from $\mathcal{N}(\bm{0}, \bm{I})$ to generate a realization of the noisy measurement variable $\Y$. Figure \ref{fig:rect_dataset} depicts a few of the paired samples $(\x,\y)$ that form the dataset. 
\begin{figure}[t]
\centering
\includegraphics[width=0.8\linewidth]{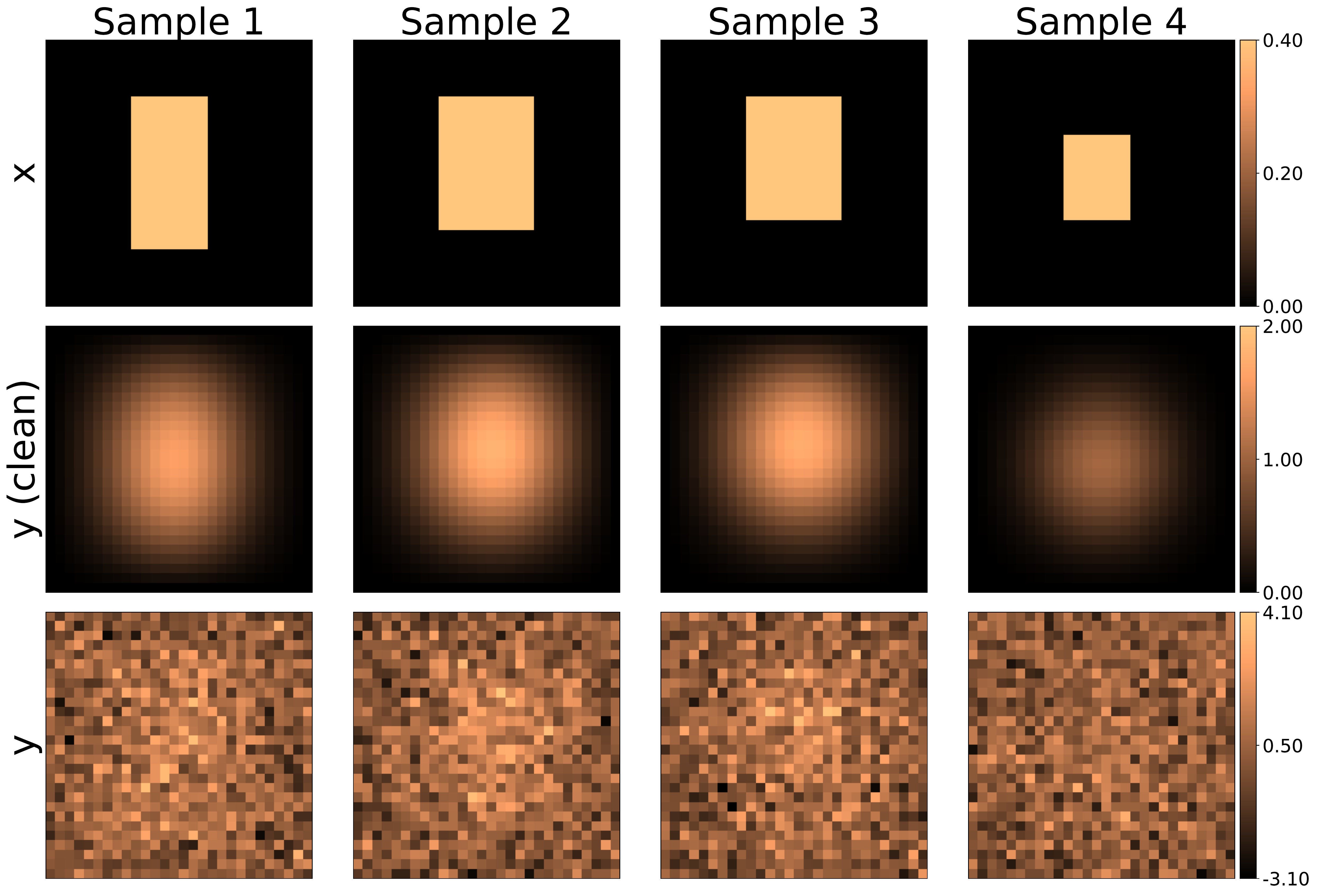}
\caption{Samples from the rectangular prior dataset for the inverse heat equation}
\label{fig:rect_dataset}
\end{figure}

\begin{figure}[t]
\centering
\includegraphics[width=0.5\linewidth]{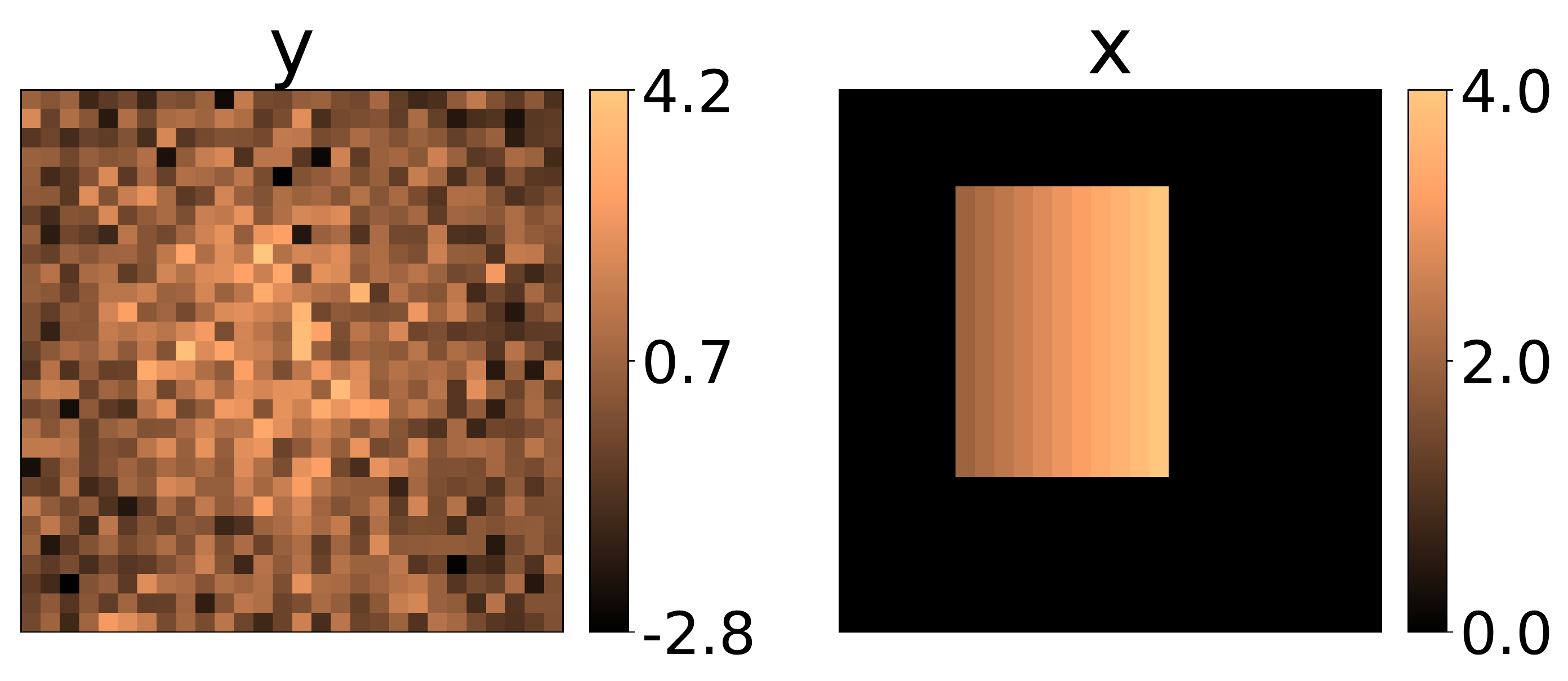}
\caption{Test $(\x,\y)$ sample for inverse heat equation}
\label{fig:test_sample_heat}
\end{figure}
\begin{figure}[t]
\centering
\includegraphics[width=0.8\linewidth]{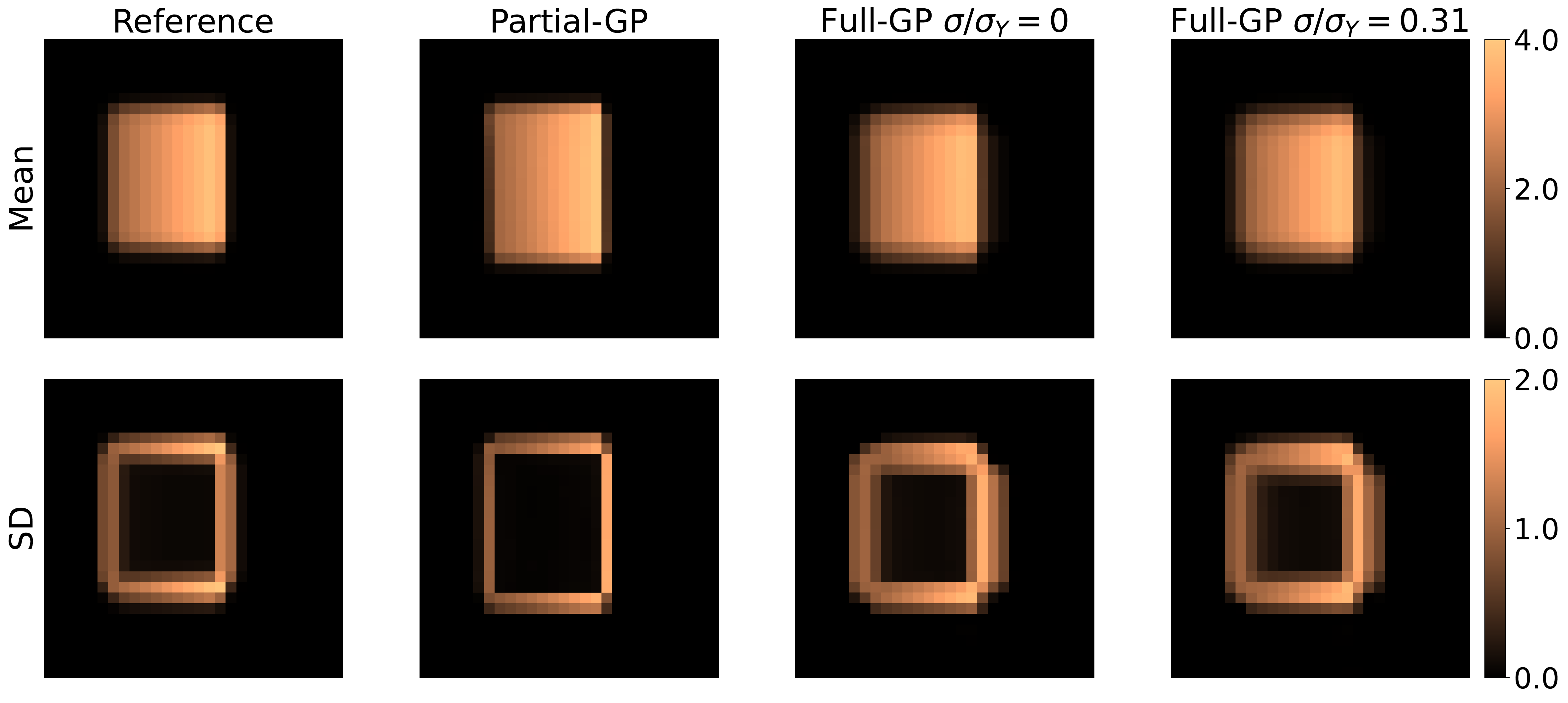}
\caption{Comparing mean and SD obtained with the Partial-GP and Full-GP models to the reference fields for the inverse heat equation}
\label{fig:compare_stats_heat}
\end{figure}
\begin{table}[b]
  \caption{$L^2$ error in posterior statistics obtained using different types of cWGANs for the inverse heat equation}  
  \label{tab:stat_errs_heat}
  \centering
  \begin{tabular}{lccc}
    \toprule
     \multirow{2}{*}{\textbf{\makecell{Posterior\\Statistics}}} & \multirow{2}{*}{\textbf{Partial-GP}} & \multicolumn{2}{c}{\textbf{Full-GP}} \\
     \cline{3-4}
     & & $\sigma/\sigma_Y=0$ & $\sigma/\sigma_Y=0.31$ \\
     \toprule
     Mean & 0.441 & 0.420 & 0.406\\
     SD   & 0.460 & 0.454 & 0.435\\
    \bottomrule
  \end{tabular}
\end{table}
Based on the experiments performed in \cite{ray2022}, we consider cWGAN models for the Partial-GP and Full-GP approaches with $\Nz = 3$. The details about the GAN architectures are available in Appendix \ref{app:archs_heat}. Both cWGANs are trained on a dataset with 10,000 training sample pairs. After a thorough sweep over the the gradient-penalty parameter $\lambda$, we found $\lambda=10$ for the Partial-GP approach and $\lambda=0.1$ for the Full-GP approach to be optimal. Both cWGANs are tested on an unseen pair shown in Figure \ref{fig:test_sample_heat}. Following \cite{ray2022}, the reference posterior pixel-wise mean and standard deviation (SD) of $\x$ given $\y$ for this test sample (depicted in the first column of Figure \ref{fig:compare_stats_heat}) is calculated using Monte Carlo sampling of the random variables $\xi_i$. As can be seen in Figure \ref{fig:compare_stats_heat} the pixel-wise statistics are captured better with the Full-GP approach compared to the Partial-GP approach. This is corroborated by the $L^2$ error of these statistics listed in Table \ref{tab:stat_errs_heat}, with the smallest errors observed for $\sigma/\sigma_Y=0.31$, where $\sigma$ is the sampling parameter for $\mu_{\Y_\sigma}$ and $\sigma_Y=1.02$ is the mean pixel-wise standard deviation over the $\y$ samples in the training set. Note that the $L^2$ error here is evaluated as
\[
\| \text{Mean}_\text{pred} - \text{Mean}_\text{ref} \|_2 = \frac{1}{\Nx \Ny} \sum_{i}^{\Nx} \sum_{j}^{\Ny} \left([\text{Mean}_\text{pred}]_{i,j} - [\text{Mean}_\text{ref}]_{i,j} \right)^2,
\]
and similarly for the standard deviation. We also perform a sweep over $\sigma$ to see how this sampling parameter influences the errors. As shown in Figure \ref{fig:sigma_heat}, the errors are comparable for $\sigma/\sigma_Y < 0.4$ while being smaller than those observed with the Partial-GP approach. As expected, the errors increase significantly for larger values of the sampling parameter.



\begin{figure}[t]
\centering
\includegraphics[width=\linewidth]{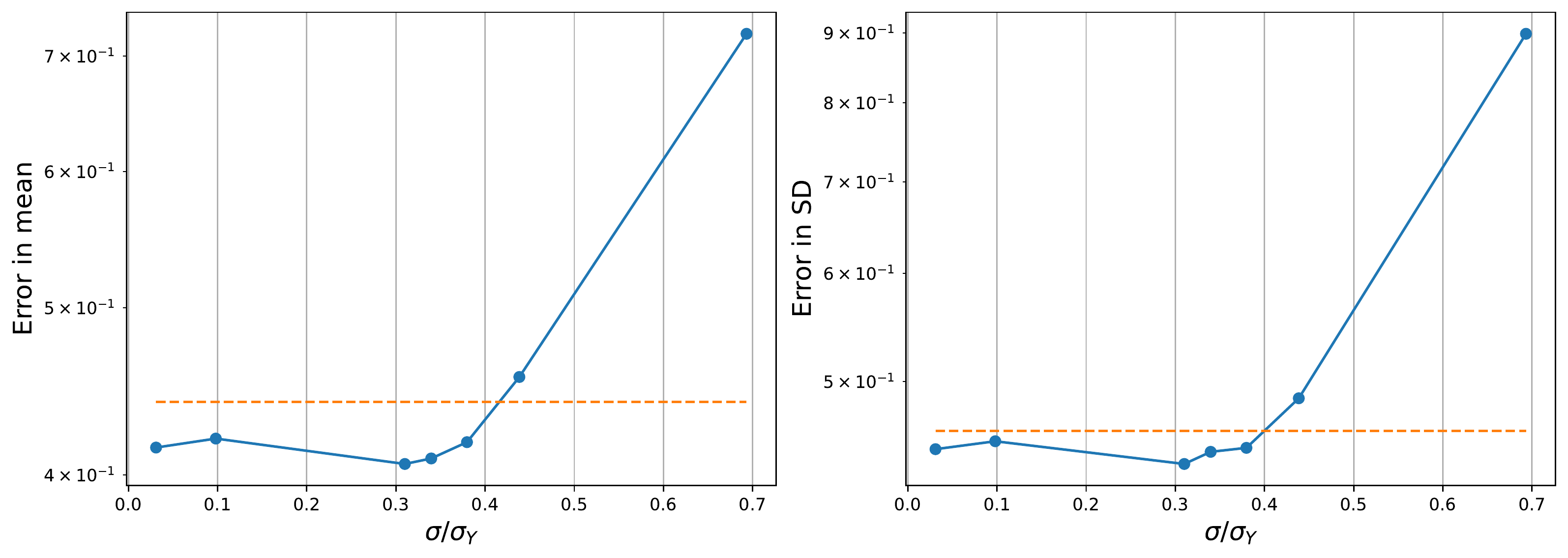}
\caption{$L^2$ errors in posterior mean and SD with Full-GP approach as $\sigma$ is varied. The dashed horizontal lines indicate the errors with the Partial-GP model}
\label{fig:sigma_heat}
\end{figure}

\subsection{Inverse Helmholtz equation}

Lastly, we use cWGANs with Full-GP to solve the inverse Helmholtz equation, which is extensively used to model the propagation of waves in elastography applications \cite{mclaughlin2006shear,zhang2012solution}. In such applications, wave fields are measured and used to infer the mechanical properties of the propagation medium. This relation is given by the Helmholtz equation:
\begin{equation}\label{eqn:pde_helmholtz} 
    -\omega^2 ( \hat{\mathbf{u}}_R + i\hat{\mathbf{u}}_I ) 
    - \nabla \!\cdot\! \Bigl(\frac{G}{\rho} (1+ i \alpha \omega) \nabla (\hat{\mathbf{u}}_R + i\hat{\mathbf{u}}_I ) \Bigl) = 0,
    \forall \ (\hat{\mathbf{u}}_R, \hat{\mathbf{u}}_I, G) \in [0,1]^2 \times [0,1]^2 \times [0,1]^2 ,
\end{equation}
where $\omega$ denotes the wave propagation frequency, $\hat{\mathbf{u}}_R$ and $\hat{\mathbf{u}}_I$ denote the real and imaginary components of the wave amplitude field at frequency $\omega$, $G$ denotes the shear modulus field, $\alpha$ is the wave dissipation coefficient and $\rho$ denotes the density. The physical domain of interest is 1.75 mm $\times$ 1.75 mm. However, we model a larger domain that includes the domain of interest to allow for wave dissipation and avoid reflections: the left edge is padded by 2.6 mm, 1.75 mm is added on the top and bottom edges, but the right edge is not padded. We impose the following boundary condition on the right boundary: $\hat{\bm{u}}_R = 0.02$ mm and $\hat{\bm{u}}_I = 0$. The boundary conditions on all other boundaries of the expanded domain are $\hat{\bm{u}}_R = \hat{\bm{u}}_I = 0$. Further, we assume that $\alpha = 5 \times 10^{-5}$ and $\rho = 1000 $ kg/m\textsuperscript{3} are constant over the entire physical domain of this problem. Now, the inverse problem we wish to solve here consists of inferring the shear modulus $G$ from the wave amplitude fields $\hat{\mathbf{u}}_R$ and $\hat{\mathbf{u}}_I$.

\begin{figure}[t]
\centering
\includegraphics[width=0.4\linewidth]{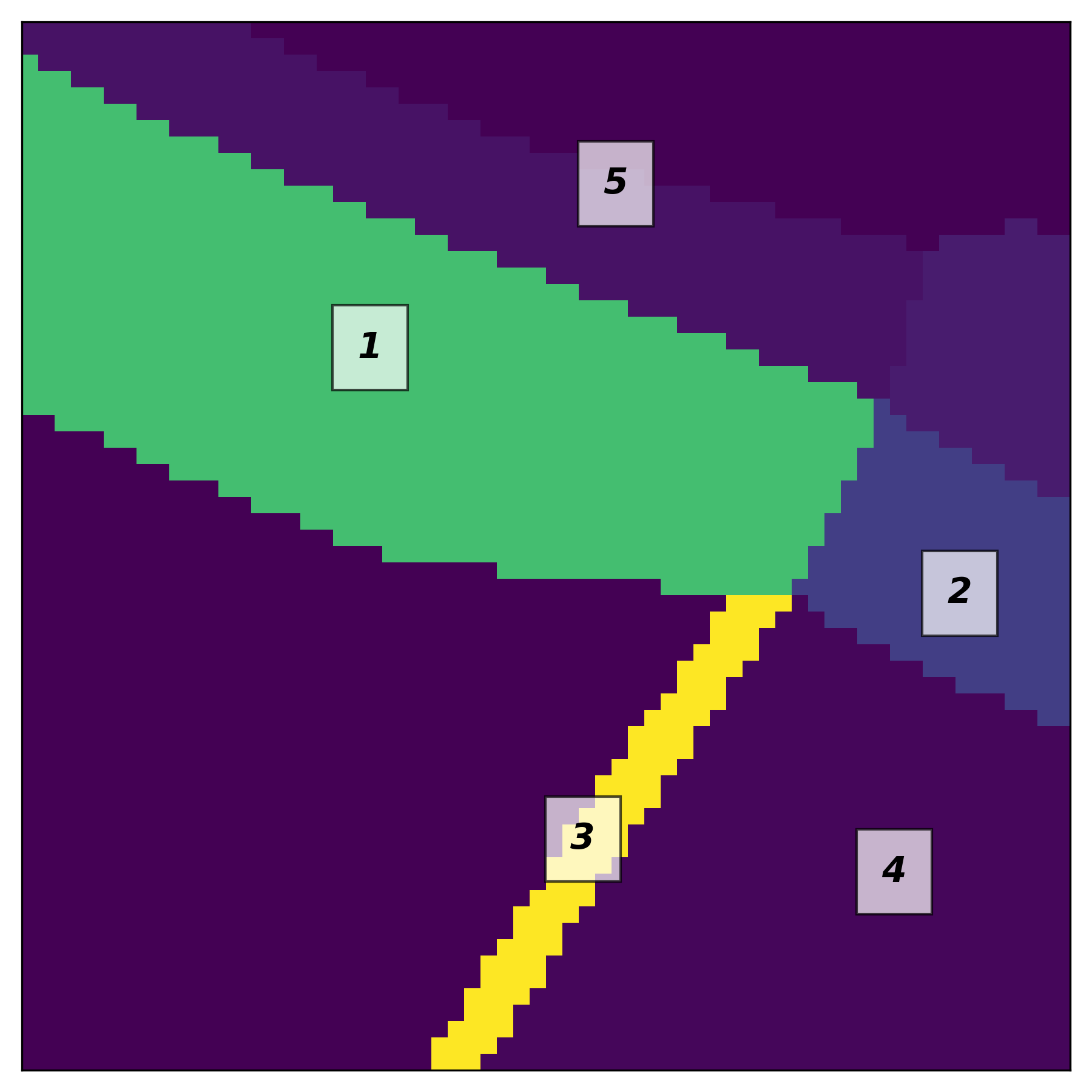}
\caption{A realization of the optic nerve head (ONH) samples from the prior measure showing the (1) sclera, (2) lamina cribrosa, (3) pia matter, (4) optic nerve, and (5) retina}
\label{fig:app_ONH}
\end{figure}
In this particular example, we simulate the case when the tissue displacement within a human 
optic nerve head (ONH) is measured using ultrasound waves.  
\Cref{fig:app_ONH} shows the typical geometry of an ONH considered in this example and delineates its various parts. The prior distribution $\mu_\X$ controls the geometry of the optic nerve head (ONH) and its shear modulus field $G$. \Cref{tab:app_ONH_rand} lists the 16 random variables that make up the prior distribution; these are adapted from the literature \cite{sigal2009interactions, csahan2019evaluation, qian2021ultrasonic, zhang2020vivo}.
\begin{table}[t]
  \caption{Random variables comprising the prior measure for the optic nerve head (ONH)} 
  \label{tab:app_ONH_rand}
  \centering
  \begin{tabular}{lc}
    \toprule
     Parameter & Definition \\
     \toprule
     Width of lamina cribrosa (mm) & $\mathcal{U}(1.1,2.7)$ \\
     Thickness of lamina cribrosa (mm) & $\mathcal{U}(0.16,0.44)$ \\
     Radius of lamina cribrosa (mm) & $\mathcal{U}(1.0,5.0)$ \\
     Thickness of the sclera (mm) & $\mathcal{U}(0.45,1.15)$ \\
     Radius of the sclera and retina (mm) & $\mathcal{U}(1.0,5.0)$ \\
     Thickness of retina (mm) & $\mathcal{U}(0.20,0.40)$ \\
     Width of optic nerve (mm) & $\mathcal{U}(0.20,0.40)$ \\
     Radius of optic nerve (mm) & $\mathcal{U}(1.65,3.65)$ \\
     Thickness of pia matter (mm) & $\mathcal{U}(0.06,0.10)$ \\
     Optic nerve shear modulus (kPa) & $\mathcal{N}(9.8,3.34^2)^*$ \\
     Sclera shear modulus (kPa) & $\mathcal{N}(125,5^2)^*$ \\
     Pia matter shear modulus (kPa) & $\mathcal{N}(125,50^2)^*$ \\
     Retina shear modulus (kPa) & $\mathcal{N}(9.8,3.34^2)^*$ \\
     Lamina cribrosa shear modulus (kPa) & $\mathcal{N}(73.1,46.9^2)^*$ \\
     Background shear modulus (kPa) & $0.1$ \\
     Rotation of the geometry (rad) & $\mathcal{U}(-\pi/12,\pi/12)$ \\
    \bottomrule
    \multicolumn{2}{c}{\small$*$In this example, the normal distributions $\mathcal{N}(0, \varsigma^2)$ are truncated to have support between $(0, 2\varsigma]$} \\
  \end{tabular}
\end{table}

Similar to the previous example, we discretize the shear modulus field $G$ over a $64 \times 64$ Cartesian grid and denote it by $\X$. The real and imaginary wave amplitude fields, $\hat{\mathbf{u}}_R$ and $\hat{\mathbf{u}}_I$, respectively, are also represented on $64 \times 64$ Cartesian grids. We solve the Helmholtz equation in \Cref{eqn:pde_helmholtz} using FEniCS \cite{alnaes2015fenics} and add isotropic Gaussian noise to the measurements; the added noise has standard deviation equal to 4\% of the maximum of $\hat{\mathbf{u}}_R$ and $\hat{\mathbf{u}}_I$ across all training samples. We denote the noisy $\hat{\mathbf{u}}_R$ and $\hat{\mathbf{u}}_I$ as $\y_1$ and $\y_2$, respectively; 
the random variable $\Y$ consists of $\y_1$ and $\y_2$. Finally, we construct a training dataset that consists of 12,000 samples of $\x$ and $\y$ pairs; three such pairs are shown in \Cref{fig:invhelm_trainsamples}. Reference posterior statistics, which allow for the fair comparison of cWGANs, are not available for this problem, so we perform a qualitative study of the performance of the cWGAN with Full-GP. We use the same hyper-parameters for the cWGANs in this example as the inverse heat conduction example; see \labelcref{app:archs} for more details. 

\begin{figure}[H]
\centering
\includegraphics[width=\linewidth]{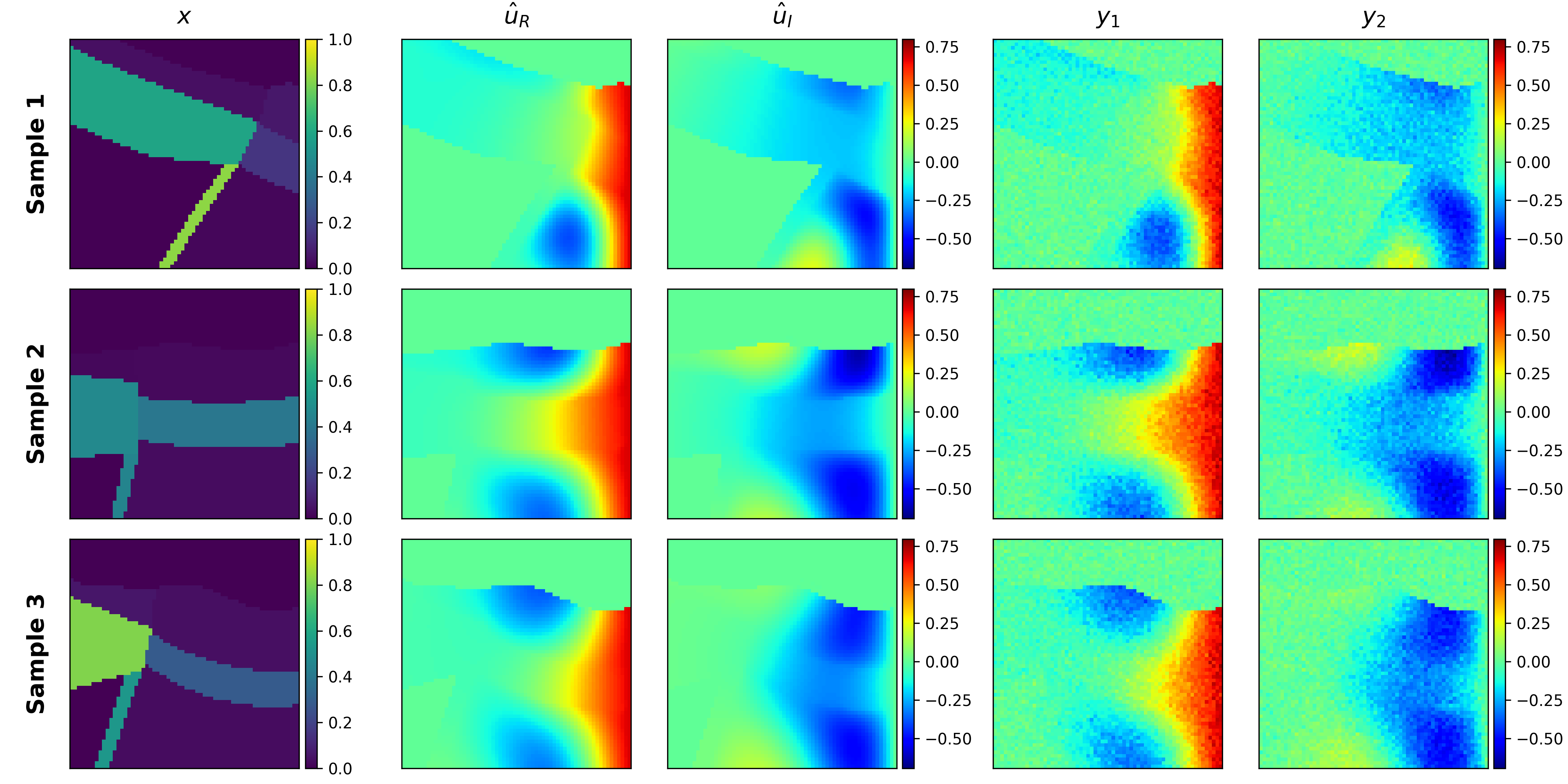}
\caption{Realizations from the prior distribution of $\X$ for the ONH shear modulus field (column 1), and corresponding wave amplitudes (columns 2 and 3) and realizations of $\Y$ (columns 4 and 5)} 
\label{fig:invhelm_trainsamples}
\end{figure}

Figure \ref{fig:invhelm_predictions} shows the predictions of the cWGAN for the three test samples shown in Figure \ref{fig:invhelm_trainsamples}. For each test sample, columns 1 and 2 show the noisy $\hat{\bm{u}}_R$ and $\hat{\bm{u}}_I$ measurements, respectively; column 3 shows the ground truth shear modulus $\x$; columns 4 and 5 show the corresponding posterior mean; and, columns 6 and 7 show the posterior standard deviation. Results are reported for two values of the sampling parameter for $\mu_{\Y_{\sigma}}$: $\sigma = 0$ (columns 4 and 6 of \Cref{fig:invhelm_predictions}) and $\sigma/\sigma_Y = 0.31$ (columns 5 and 7 of \Cref{fig:invhelm_predictions}); the latter value of $\sigma/\sigma_y$ is reused from the inverse heat conduction problem and $\sigma_y=0.105$ is the mean pixel-wise standard deviation over the $\y$ samples in the training set. The posterior statistics are estimated using $3,000$ samples. The posterior mean estimated using cWGANs with Full-GP is similar to the ground truth for all three test samples and for both values of $\sigma$ considered here. Moreover, we observe large standard deviations in the posterior distribution in areas where we expect the predictions to be uncertain, such as the edges of the ONH. However, the predicted posterior standard deviation when $\sigma/\sigma_Y = 0.31$ is comparatively larger than the predicted posterior standard deviation when $\sigma = 0$. This may be attributed to the additional sampling variance injected into $\y$ when $\sigma/\sigma_Y = 0.31$.

\begin{figure}[H]
\centering
\includegraphics[width=\linewidth]{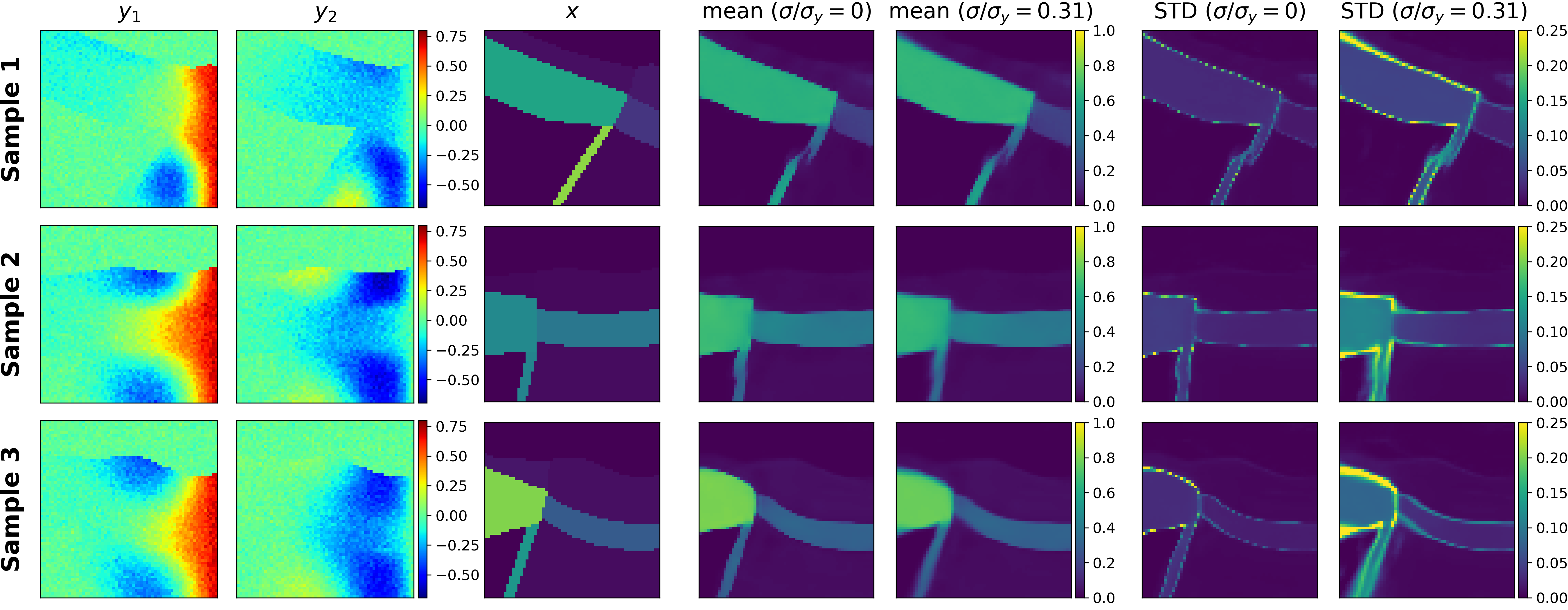}
\caption{Predictions from the cWGAN with Full GP. Columns 1 \& 2 are the measurements $\y$ and column 3 is  the true shear modulus $\x$; columns 4 and 6 show the component-wise mean and standard deviation, respectively, of $3,000$ samples obtained from Full-GP cWGAN with $\sigma/\sigma_y = 0$; columns 5 and 7 are the corresponding results obtained with $\sigma/\sigma_y = 0.31$}
\label{fig:invhelm_predictions}
\end{figure}


\section{Conclusions and Outlook}\label{sec:conclusion}
In this manuscript a novel version of the cWGAN is developed for solving probabilistic inverse problems where the forward operator is constrained by a physics-based model. The approach is particularly useful for problems where the prior information for the inferred vector is known through samples, and the forward operator is known only as a black-box.  The cWGAN developed in this work differs from earlier versions in that its critic is required to be 1-Lipschitz with respect to both the inferred and the measurement vectors and not just the former. This leads to a loss term with a gradient penalty for the critic which is computed using all of its arguments. This simple change has a significant impact on the cWGAN. It allows us to prove that the conditional distribution learned by the cWGAN weakly converges to the true conditional distribution. It also leads to a new sampling strategy, wherein the output of the generator is computed for measurements sampled from a Gaussian distribution tightly centered around the true measurement. Through numerical examples it is demonstrated that the new cWGAN is more accurate than its predecessor, and the that new sampling strategy helps in further improving its performance. The application of this approach to a challenging inverse problem that is by applications motivated in biomechanics is also described.  

The ideas described in this manuscript may be extended in several interesting directions. These include, (i) their application to more challenging and complex inference problems; (ii) their application to probabilistic operator networks \cite{lu2021learning, li2020fourier, patel2022variationally} that map functions of one class to another (initial state to final state, or vice-versa); (iii) the use of more complex network architectures like transformers and self-attention networks~\cite{vaswani2017attention} in the generator and critic networks of the cWGAN.

\section{Acknowledgments}

The authors acknowledge support from ARO, USA grant W911NF2010050 and from the Airbus Institute for Engineering Research at USC. The authors acknowledge the Center for Advanced Research Computing (CARC) at the University of Southern California, USA for providing computing resources that have contributed to the research results reported within this publication.

\bibliographystyle{elsarticle-num-names} 
\bibliography{references}

\appendix
\setcounter{figure}{0}
\setcounter{table}{0}
\renewcommand{\thesubsection}{\Alph{section}\arabic{subsection}}

\section{GAN architectures and experimental setup}\label{app:archs}

The generator and critic architectures for the experiments considered in \Cref{sec:results} are described in this section. Other key hyper-parameters are listed in Table \ref{tab:hparams}. All networks were trained using Adam optimizer with parameters $\beta_1 = 0.5$, $\beta_2=0.9$.

\begin{table}[!htbp]
\renewcommand{\arraystretch}{1.5}
\centering
\caption{Hyper-parameters associated with the cWGANs for various numerical examples}
\label{tab:app_hyperparameters}
\begin{tabular}{c c c c c c}
\toprule
\multirow{3}{*}{Hyperparameter}& \multicolumn{5}{c}{Experiment} \\
\cline{2-6}
 & \multicolumn{3}{c}{Illustrative examples} & \multirow{2}{*}{\makecell{Inverse heat\\conduction}} & \multirow{2}{*}{\makecell{Inverse Helmholtz\\equation}} \\
\cmidrule{2-4}
 &   Tanh+$\Gamma$ &   Bimodal  & Swissroll & &\\ 
\midrule
Training samples & 2,000 & 2,000 & 2,000 & 10,000 & 12,000 \\
$\Nx$ & 1 & 1 & 1 & $28\times28$ & $64\times64$ \\
$\Ny$ & 1 & 1 & 1 & $28\times28$ & $64\times64$ \\
$\Nz$ & 1 & 1 & 1 & 3 & 50 \\
Batch size & 2,000 & 2,000 & 2,000 & 50 & 50\\
Activation func & tanh & tanh & tanh & LReLU(0.1) & ELU \\
$N_\text{max}$ & 20 & 20 & 20 & 4 & 4\\
Max epochs & 600,000 & 600,000 & 600,000 & 500 & 4,000 \\
Learning rate & $10^{-4}$ & $10^{-4}$ & $10^{-4}$ & $10^{-3}$ & $10^{-3}$ \\
\bottomrule
\end{tabular}\label{tab:hparams}
\end{table}

\subsection{Illustrative examples}\label{app:archs_2Dexp}

The generator and critic architectures chosen to solve this problem are multilayer perceptrons (MLPs) with 4 hidden layers with $128-256-64-32$ units, respectively. 


\subsection{Inverse heat conduction}\label{app:archs_heat}
The generator and critic architectures are identical to those considered for the inverse heat conduction problem in \cite{ray2022}. In particular, a residual block based U-Net architecture is used for the generator, with the latent information injected at every level of the U-Net through conditional instance normalization. The critic comprises residual block based convolution layers followed by a couple of fully connected layers. See \cite{ray2022} for precise details of the architectures, and the benefits of using conditional instance normalization, especially when the generator input $\y$ is an image while the latent variable is a vector.

\subsection{Inverse Helmholtz problem}\label{app:archs_invhelm}

For this problem, we used the same generator and critic architectures as for the inverse heat conduction problem. The only differences are regarding the different size of the inputs and using the ELU activation function instead of LReLU(0.1). We used the same hyper-parameters as in the inverse heat equation problem (see \Cref{tab:app_hyperparameters}). 




\end{document}